\documentclass{amsart}
\usepackage{amsmath}
\usepackage{caption}
\usepackage{booktabs}
\usepackage{amssymb}
\usepackage{comment}
\usepackage{color}
\usepackage{float}
\usepackage{layout}
\usepackage{amsthm}
\usepackage{bbm}
\usepackage{bm}
\usepackage{xy}
\usepackage{cleveref}
\usepackage{enumitem}
\usepackage[ruled,vlined,linesnumbered,boxed]{algorithm2e}
\usepackage{mathrsfs}
\usepackage[gen]{eurosym}
\usepackage{graphicx}
\usepackage{wrapfig}
\usepackage{todonotes}
\usepackage[section]{placeins}

\newcommand{\seb}[1]{}
\newcommand{\yichao}[1]{}
\newcommand{\notesteven}[1]{}
\newcommand{\note}[1]{}

\calclayout

\theoremstyle{plain}
\newtheorem{theorem}{Theorem}[section]

\newtheorem{lemma}[theorem]{Lemma}

\newtheorem{definition}[theorem]{Definition}

\newtheorem{assumption}[theorem]{Assumption}
\theoremstyle{remark}
\newtheorem{remark}[theorem]{Remark}

\numberwithin{equation}{section}

\newcommand{\mcA}{{\mathcal{A}}}

\newcommand{\mcK}{{\mathcal{K}}}

\newcommand{\mcP}{{\mathcal{P}}}

\newcommand{\Ff}{\mathcal{F}}

\newcommand{\Gg}{\mathcal{G}}

\newcommand{\PP}{\mathbb{P}}
\newcommand{\EE}{\mathbb{E}}

\newcommand{\RR}{\mathbb{R}}

\newcommand{\mfT}{{\mathfrak{T}}}
\newcommand{\mfN}{{\mathfrak{N}}}

\crefname{algorithm}{Algorithm}{Algorithm}

\newcommand\myprime{\mkern-3.5mu\raise0.6ex\hbox{$\scriptstyle\prime$}}

\makeatletter
\newcommand*\bigcdot{\mathpalette\bigcdot@{.8}}
\newcommand*\bigcdot@[2]{\mathbin{\vcenter{\hbox{\scalebox{#2}{$\m@th#1\bullet$}}}}}
\makeatother

\begin{document}
\title[Deep Learning for PA-MFGs]{Deep Learning for Principal-Agent Mean Field Games
}

\begin{abstract}
Here, we develop a deep learning algorithm for solving Principal-Agent (PA) mean field games with market-clearing conditions -- a class of problems that have thus far not been studied and one that poses difficulties for standard numerical methods. We use an actor-critic approach to optimization, where the agents form a Nash equilibria according to the principal's penalty function, and the principal evaluates the resulting equilibria. The inner problem's Nash equilibria is obtained using a variant of the deep backward stochastic differential equation (BSDE) method modified for McKean-Vlasov forward-backward SDEs that includes dependence on the distribution over both the forward and backward processes. The outer problem's loss is further approximated by a neural net by sampling over the space of penalty functions. We apply our approach to a stylized PA problem arising in Renewable Energy Certificate (REC) markets, where agents may rent clean energy production capacity, trade RECs, and expand their long-term capacity to navigate the market at maximum profit. Our numerical results illustrate the efficacy of the algorithm and lead to interesting insights into the nature of optimal PA interactions in the mean-field limit of these markets.
\end{abstract}


\author[S. Campbell]{Steven Campbell}
\thanks{SC would like to acknowledge support from the Natural Sciences and Engineering Research Council of Canada (Alexander Graham Bell Canada Graduate Scholarship Application No. CGSD3-535625-2019).}
\address{
Department of Statistical Sciences, University of Toronto.
}
\email{steven.campbell@mail.utoronto.ca}



\author[Y. Chen]{Yichao Chen}
\address{Department of Mathematics, University of Toronto.}
\email{yichao.chen@mail.utoronto.ca}

\author[A. Shrivats]{Arvind Shrivats}
\address{Department of Operations Research \& Financial Engineering, Princeton University.}
\email{shrivats@princeton.edu}

\author[S. Jaimungal]{Sebastian Jaimungal}
\thanks{SJ would like to acknowledge support from the Natural Sciences and Engineering Research Council of Canada (grants RGPIN-2018-05705, RGPAS-2018-522715).}
\address{Department of Statistical Sciences, University of Toronto.}
\email{sebastian.jaimungal@utoronto.ca}
\urladdr{http://sebastian.statistics.utoronto.ca}

\date{\today}

\maketitle %
\section{Introduction}

The study of equilibria in systems of interacting agents is a field that has commanded considerable attention in recent years. Understanding  how firms optimize their strategic decision making while interacting with other optimizing firms has broad applicability to areas in nature, social sciences, and finance.
In this work, we apply advances in deep learning (DL) to propose an algorithm that provides a solution to a broad class of complex problems regarding the equilibrium of a large number of small competing interactive agents, commonly known as mean field games (MFGs). Moreover, the class of problems we study incorporates additional layers of complexity, as we consider equilibria within a principal-agent (PA) formulation, otherwise known as a principal-agent mean field games (PA-MFGs). Furthermore, the problems that we study may include market clearing conditions which introduces an additional interaction into the PA-MFG problems that did not exist before. 

Specifically, we consider a PA-MFG problem where the principal induces an infinite collection of agents to work on their behalf through an incentivized payment structure. Given the principal's incentivized payment structure, the agents form a Nash equilibrium to optimize their goals (the inner problem). The principal then aims to find a payment structure that best incentivizes the agents for the principal's needs (the outer problem). This is a highly non-trivial problem, both in finding the Nash equilibrium and in finding the principal's optimal incentive structure.

Our methodology utilizes deep neural networks (DNN) to solve a McKean-Vlasov Forward Backward Stochastic Differential Equation (MV-FBSDE) that is associated with the solution of the inner problem's MFG. Simultaneously, we dynamically estimate the principal's loss surface and implement a gradient descent algorithm to  solve the outer problem. The DNN parametrization of the MV-FBSDE's solution allows us to solve complex and high dimensional problems that classical numerical approaches, such as Least Squares Monte Carlo (LSMC) or iterative semi-analytic schemes, fail to provide efficient solutions to. Moreover, the dynamic estimation and concurrent updates for the inner MFG problem and outer PA problem expedites convergence to a solution.

An application of this particular problem is readily found in the world of environmental regulation, and particularly emissions markets. Examples of these include carbon cap-and-trade or Renewable Energy Certificate (REC) markets. Indeed, regulators may be viewed as principals who choose a market design to induce some combination of revenue generation and environmental compliance, with the regulated power generators being agents who interact with one another through the market and attempt to navigate the regulatory environment to maximize profit. Thus, the latter half of this paper discusses the PA-MFG problem in a REC market and presents the results of our proposed algorithm for a synthetic principal and set of agents.

\cite{shrivats2020optimal, carmona2010market, carmona2009optimal} study optimal behavior within REC (and C\&T) markets at the individual firm level, and \cite{seifert2008dynamic} from the perspective of a social planner. These market designs have been discussed in some of the aforementioned works (\cite{carmona2010market, carmona2009optimal} in particular) as well as in \cite{coulon2015smart, khazaei2017adapt}. To the best of the authors' knowledge, however, this is the first work to explicitly consider a PA-MFG approach to efficient emission market design. 

PA games in general are well-studied. Canonically, they consider a single principal and agent, with \cite{holmstrom1987aggregation} being an early rigorous contribution to this field. Recent years have seen renewed efforts, thanks in large part to \cite{sannikov2008continuous} who investigates a class of infinite-horizon PA games where the the principal rewards the agent continuously. \cite{sannikov2008continuous} leverages the dynamic nature of the agent’s value function arising from the dynamic programming principle. This allows the principal’s problem to be formulated as a tractable optimal control
problem, and this approach was  made rigorous and expanded upon in \cite{cvitanic2018dynamic} and \cite{cvitanic2017moral}, using second order BSDEs.

These works focus on the case where a single principal contracts a single agent. Many real-world situations (such as the aforementioned REC markets), however, involve a principal influencing a number of agents. With large populations of agents, finding equilibria is generally intractable, which motivates the idea of applying MFG theory to the  PA problem. This is exactly the motivation taken by \cite{elie2019tale}
and \cite{carmona2021finite}, the former in a continuous state space, and the latter in a discrete state space.

Naturally,  \cite{elie2019tale,carmona2021finite}, as well as our own, rely heavily on MFG theory. MFG theory started with the seminal works of \cite{huang2007large, huang2006large} and \cite{Lasry2006a, Lasry2006b, Lasry2007mean}. The methodology developed in these concurrent works provides approximate Nash equilibria in stochastic differential games with symmetric interactions and a large number of players. Any individual player is impacted by the others through the empirical distribution of their states, a fact which makes these problems generally intractable. The key methodological insight of MFG theory is that the problem is simplified by taking the number of players to infinity, as the empirical distribution can be replaced by a mean-field distribution. This allows for a Nash equilibrium to be found in the infinite player limit, which crucially can be proven to provide an approximate Nash equilibrium for the finite player game.

Over the past 15 years, MFGs have exploded in popularity, with numerous applications to engineering (\cite{aziz2016mean, KIZILKALE2019,Tembine2017}), economics (\cite{Gomes2014,Gomes2016}), and mathematical finance (\cite{CarmonaLacker2015, Mojtaba2015, ThesisDena2019, FirooziISDG2017, FirooziPakniyatCainesCDC2017, Cardaliaguet2018,Lehalle2019, casgrain2018mean, casgrain2020mean, Horst2018, David-Yuri2020}). There have also been efforts to apply MFGs to environmental regulation, as in \cite{shrivats2020mean}, as well as commodities (\cite{aid2017coordination, Mouzouni2019,Sircar2017, brown2017oil, ludkovski2017mean}).

The intersection of MFGs and DL has also grown recently. The probabilistic approach to solving for Nash equilibria in MFGs leads to an MV-FBSDE system. To solve FBSDE systems, the deep BSDE method \cite{han2018solving,weinan2017deep,beck2019machine} has been remarkably successful. A rigorous theoretical foundation for this method is developed in \cite{han2020convergencebsdes}. In a related setting, DL techniques have also been used to solve $N$-player stochastic differential games \cite{han2020deep,han2020convergence}. More recently, \cite{carmona2019convergence,carmona2021convergence} directly treat MFGs and mean-field type control (MFC) problems in an DL setting, propose efficient numerical algorithms for arriving at a solution, and carry out a detailed analysis of their convergence. Similarly, \cite{fouque2020deep} uses DL techniques to tackle MFC problems with delay. \cite{carmona2021deep} provides an excellent overview of techniques in this area and their applications to finance. 

The contributions of this work are threefold. Firstly, we further generalize the application of the deep BSDE method seen in \cite{han2018solving,weinan2017deep,beck2019machine,han2020convergencebsdes,carmona2019convergence,carmona2021convergence,carmona2021deep} to MV-FBSDEs of the type that have dependence on the law of the adjoint process. This arises, for example, in settings where there is market clearing in addition to equilibrium (see, e.g., \cite{shrivats2020mean}). Secondly, we develop an alternating optimization method to solve the principal's problem concurrently with the solution to the MV-FBSDEs. In the process, we introduce an equivalent formulation of the PA problem that lends itself to unconstrained optimization. To the best of our knowledge, this is the first application of DL techniques to the PA-MFG setting. Lastly, we implement our algorithm on a non-trivial PA-MFG problem of independent interest in REC markets. The result yields several interesting insights into the optimal interactions that arise in the mean-field regime.

The remainder of this paper is organized as follows. \Cref{sec:problem_formulation} provides the mathematical framework in which we operate, including the formal definition of the general principal agent problem we wish to solve. Following that, we present our algorithm in \Cref{sec:algorithm}. We conclude by discussing an application of our approach to environmental markets in \Cref{sec:example} and present the algorithm's solution to this problem in \Cref{sec:results}. 

\section{The General Principal Agent mean-field Game Problem} \label{sec:problem_formulation}

In a typical PA game, there are two classes of decision makers: (i) the agent, and (ii) the principal. The agent has the ability to influence some output process, denoted by $X$, through the modification of some control process, denoted by $\alpha$. The principal delegates the management of this output process to the agent and is assumed to derive some benefit from it. They compensate the agent at some terminal time $T$ for the output they produce. This provides an incentive for the agent to manage $X$ properly. Principals cannot observe the control of the agent. Rather, they only observe the agent's output process. This asymmetry of information leads to moral hazard. 

In this work, we assume the principal compensates the agent only through Markovian functions of the output process. 
This compensation may be negative, in which case funds flow from agent to principal. The problem the principal faces is: What is the optimal compensation structure to balance the output process that the agent manages and the amount the agent is compensated (or charged) for their efforts?

The agent wishes to maximize profit, treating the reward structure offered by the principal as exogenous. Typically, PA problems include a mechanism referred to as a reservation cost which ensures the agent receives a minimum reward for their effort. The reservation cost reflects agents preferences to work on behalf of a principal only if they can expect a minimum payment.  

A natural extension to the PA problem described above is to consider what happens when multiple agents are working on behalf of the principal -- such as those that arise in regulation and oversight. In this case, the principal (regulator) sets rules to oversee a set of agents (regulated firms) to induce desired agents' behavior.  Environmental regulation is a clear example of such a situation, and in \Cref{sec:example} we provide a specific example of it. Here, the reservation cost may represent the maximum burden the principal can place on the set of regulated agents for the regulations to be politically feasible.

In the multiple agent setting, the agents may also interact with each other, adding a complication to the optimization that an individual agent must solve to behave optimally. In essence, each agent must solve a multi-player stochastic game, which is generally intractable. Using an MFG approximation of the many-player game helps to partially resolve the intractability.

MFGs derive from the idea that stochastic games with large but finite amounts of agents may be approximated with an infinite number of agents. The MFG solution then provides strategies that result in near Nash equilibria for the finite player game. In this work, we do not concern ourselves with the translation of the infinite player solution back to a solution of the finite player problem. Rather, we consider only the infinite player problem directly.

Specifically, we consider a PA game with an infinite number of agents. Agents belong to one of a finite set of sub-populations or agent type. Within sub-populations agents are assumed to have identical preferences, while across sub-populations, they are assumed to (possibly) vary through the parameters to their cost functions. 

The advantage of considering the infinite player setting is that the MFG is significantly simpler to solve than the finite player game. This is due to the assumption that an individual agent is insignificant as the number of agents tends to infinity. Thus, we may model the agents' interactions through a mean-field distribution representing the population distribution of output processes. In contrast, finite player games require the agents to interact directly, which is significantly more complicated. 

Therefore, in this PA-MFG setting, our goals are as follows. For a given reward structure of the principal, we seek the collection of controls that results in a Nash equilibrium across agents. Similarly, we seek the reward structure which benefits the principal the most, taking into account the equilibrium that ensues for any particular admissible reward structure. We begin by formulating the agent's problem.

\subsection{The agents' problem in a generic PA-MFG} \label{sec:general_agent_problem}

We consider the agents' problem with $N \rightarrow \infty$ agents in total, each belonging to one of $K$ different sub-populations. The relevant time horizon is defined to be $\mfT:= [0,T]$. We work on the filtered probability space $(\Omega, \Ff, (\Ff)_{t \in \mfT}, \PP)$. All processes are assumed to be $\Ff$-adapted unless otherwise stated. We define the filtration specifically later in this section.

We denote the set of agent types by $\mcK := \{1, \cdots, K\}$, the subset of agents belonging to sub-population $k \in \mcK$ by $\mfN_k$, and the set of all agents by $\mfN$. From the perspective of the agents, the principal exogenously specifies the payment. 

Agent $i$, belonging to sub-population $k$, seeks to modulate their controls, denoted by $(\alpha_t^i)_{t \in \mfT}$. $\alpha_t^i$ is assumed to take values in a set $A \subset \RR^j$. Through modulating these controls, the agent modifies the output process which satisfies the stochastic differential equation (SDE)
\begin{equation}
    dX_t^{i} = b^k(t, X_t^{i}, \bm{\mu_t}, \alpha_t^i) dt + \sigma^k(t) dW_t^i; \qquad X_0^i = \xi^i \sim F(\Theta^k)
\end{equation}
Here, $W = \{W^i = (W_t^i)_{t \in \mfT}, i \in \mfN\}$ is a set of $N$ $m$-dimensional independent Wiener processes, and $W^i$ is progressively measurable with respect to the filtration $\Ff^W := (\Ff_t^W)_{t \in \mfT} \subset \Ff$. The function $b: \mfT \times \RR^d \times \mathcal{P}(\RR^d) \times A: \rightarrow \RR^d$ is a deterministic and measurable function which specifies the drift of the output process, and the function $\sigma: \mfT \rightarrow \RR^{d \times m}$ is a deterministic and measurable function which specifies the diffusion coefficient of the output process. $\Theta^k$ is a vector valued set of parameters which may vary across sub-populations.

We make one brief notational remark. Where appropriate, we use a superscript $\cdot^{(k)}$ rather than a superscript $\cdot^i$ to denote quantities relating to a representative (unspecified index) agent belonging to sub-population $k$. For example $X_t^{(k)}$ represents the output process at time $t$ for an representative (generic) agent from sub-population $k$. We also define $\bm{X_t} := (X_t^{(k)})_{k \in \mcK}$ as the collection of $K$ output processes representing the joint output processes from representative agents from each of the sub-populations.

Next, we make the following definition of the mean-field distribution of states $\bm{\mu}$.
\begin{definition} [Mean-Field Distribution of States] \label{def:mean_field_distribution} 
In the infinite-player setting, we denote the mean-field distribution of the output process for agents in sub-population $ k \in \mcK$ by $\mu^{(k)}$. Specifically, we introduce the flow of measures
\begin{equation}
    \mu^{(k)} = (\mu_t^{(k)})_{t \in\mfT},
    \qquad \mu_t^{(k)} \in \mcP(\RR^d), \;\forall t\in\mfT
\end{equation}
for all $k \in \mcK$,  where $\mathcal{P}(\RR^d)$ represents the space of probability measures on $\RR^d$, such that $\mu_t^{(k)}(A)$ is the probability that a representative agent from sub-population $k$ has an output process belonging to the set $A \in \mathcal{B}(\RR^d)$, at time $t$. 
Furthermore, we define
\begin{equation}
    \bm{\mu} = (\{\mu_t^{(k)}\}_{k \in \mcK})_{t \in \mfT}
\end{equation}
to be the flow of the collection of all mean-field measures.
\end{definition}
We also make an assumption about the initial condition of the output process.
\begin{assumption} \label{IntialStateAss}
The initial states $\{\xi^i\}_{i \in \mfN_k}$ 
are identically distributed, mutually independent, and independent of $\mathcal{F}^W$. Moreover, there exists a $c>0$ such that $\mathbb{E}[\Vert \xi^i\Vert^2] \leq c < \infty $, for all $i \in\mfN_k$ and $k \in \mcK$.
\end{assumption}

Next, we denote by 
\begin{equation}
    \Gg_t^i := \sigma\left((X_u^i)_{u \in [0,t]}\right),
\end{equation}
the filtration that the $i$-th agent must adapt their strategy to, which is the sigma field generated by the $i$-th firm's output process. Note that we assume that all firms have knowledge of the initial distribution (but not the actual value of) other firms' output processes. The full filtration $\Ff = (\Ff_t)_{t\geq 0}$ is defined as $\Ff_t = \bigvee_{i \in \mfN} \Gg_t^i$. The set of square integrable controls is defined as
\begin{equation}
    \mathbb{H}_t^2 := \left\{\alpha : \Omega \times \mfT \rightarrow A \subset \RR^j\; \bigg\vert \; \EE\left[{\int_0^T} (\alpha_t)^2\, dt \right] < \infty\right\}.
\end{equation}
The admissible set of controls for the agent are defined in the assumption below.
\begin{assumption} \label{ass: MinorContrAction} The set of admissible controls for firm $i\in\mfN$  is
\begin{align}
    \mcA^i := \left\{ \alpha \in \mathbb{H}_t^2\,\, \text{s.t. }  \alpha \text { is $\Gg^i$-adapted}\,\, \right\}. \label{firmsAdmissibleCntrl}
\end{align}
\end{assumption}
Each agent chooses their controls with the goal of minimizing the cost they incur through their acceptance of the principal's contract. Specifically, for a fixed $\bm{\mu}$, the $i$-th firm (belonging to sub-population $k$) aims to minimize the cost functional $J^{A,i}: \mcA^i \rightarrow \RR$, where
\begin{equation}
    J^{A,i}(\alpha;\bm{\mu}) = \EE\left[\int_0^T f^k(t, X_t^i, \bm{\mu_t}, \alpha_t^i)dt + g(X_T^i)\; \bigg\rvert \;\Gg_0^i\right].
\end{equation}
The functions $f^k$ and $g$ are commonly referred to as the running costs and terminal costs, respectively. The terminal cost can be viewed as the control of the principal, but in the context of the agents' problem, it is exogenous and fixed. This terminal cost may also be a function of the terminal mean-field distribution, but we focus on the case when it does not for simplicity.

Agents within a sub-population have the same cost parameters, and consequently,  those agents act, in equilibrium, in a similar manner. Each individual agent's strategy, however, is adapted to their own inventory and as such, agents' strategies are not identical, even within the same sub-population.

As previously stated, all agents seek to minimize their own costs, and we seek  the optimal control for all agents simultaneously. Specifically, we seek a Nash equilibrium. This is a collection of controls $\{(\alpha^{i, \star}) \in \mcA^i: i \in \mfN\}$ and mean-field flow $\bm{\mu^\star}$ (induced by the collection of controls) such that for any other admissible control $\alpha^i$ we have:
\begin{align}
    J^{A, i}(\alpha^{i, \star}, \bm{\mu^\star}) &\leq J^{A, i}(\alpha^{i}, \bm{\mu^\star}) \qquad \forall i \in \mfN_k, \, \forall k\in\mcK, \label{eq:MFG_optimal} \\
    \PP_{X_t^{(k)}} &= \mu_t^{(k)} \qquad \forall k \in \mcK. \label{eq:MFG_consistency}
\end{align}
Equation \eqref{eq:MFG_optimal} describes the optimality of the collection of controls, and indicates that no agent can improve their position by unilaterally deviating from it. Equation \eqref{eq:MFG_consistency} ensures that the collection of controls implies a controlled output process for the representative agent (recall agents are assumed to have identical preferences within sub-populations) from each sub-population that is consistent with the mean-field distribution used in the value function.

For an exogenous $g$,  let $V^{A,k}(g) =\inf_{\alpha \in \mcA^\cdot} J^{A, \cdot}(\alpha; \bm{\mu^\star})$ represent the optimally controlled value function for a representative agent in sub-population $k$, $\forall \;k\in\mcK$.

To solve \eqref{eq:MFG_optimal} and \eqref{eq:MFG_consistency}, we  ultimately express the problem in terms of an MV-FBSDE. There are many possible approaches to obtaining such an MV-FBSDE. Two common ones are applying the so-called probabilistic method of MFGs (see \cite{carmona2018probabilistic} for a detailed pedagogical treatment) or tools from variational analysis (see e.g., \cite{shrivats2020mean, casgrain2018mean,firoozi2020convex}).

In this work, we assume that the agents' MFG problem has already been cast into its MV-FBSDE form.
\begin{assumption}\label{ass:MVFBSDE}
We assume the Nash equilibria characterized by \eqref{eq:MFG_optimal} and \eqref{eq:MFG_consistency} may be expressed through the solution of an MV-FBSDE. That is, for each $i \in \mfN_k$,
\begin{align}
    dX_t^i &= \varphi^k(t, X_t^i, \PP_{\bm{X_t}} , Y_t^i, \PP_{\bm{Y_t}})dt + \sigma^k(t) dW_t^i; \qquad X_0^i = \xi^i \sim F(\Theta^k) \label{eq:MV-FBSDE_fwd}\\ 
    dY_t^i &= \rho^k(t, X_t^i, \PP_{\bm{X_t}}, Y_t^i, \PP_{\bm{Y_t}})dt + Z_t^i dW_t^i; \qquad Y_T^i = \partial_x g(X_T^i). \label{eq:MV-FBSDE_bwd}
\end{align}
Specifically, we assume solving \eqref{eq:MV-FBSDE_fwd}-\eqref{eq:MV-FBSDE_bwd} results in a triple $(X_t^{i,\star}, Y_t^{i,\star}, Z_t^{i,\star})$ such that the optimal $\alpha^{i, \star}$ can be recovered from the solution and $\alpha^{i, \star}$ solves \eqref{eq:MFG_optimal} with $\bm{\mu^\star} := \PP_{\bm{X_t^\star}}$ (which naturally solves \eqref{eq:MFG_consistency}). Here, $\bm{X_t^\star} := (X_t^{\star, (k)})_{k \in \mcK}$, is the collection of $K$ output processes from representative agents across the sub-populations which solve the MV-FBSDE above.
\end{assumption}

\begin{remark}
When there is no market clearing, $\varphi^k(t, X_t^i, \PP_{\bm{X_t}}, Y_t^i, \PP_{\bm{Y_t}})=b^k((t, X_t^i, \PP_{\bm{X_t}},\hat{\alpha}^k(t, X_t^i, \PP_{\bm{X_t}}, Y_t^i))$ and $\rho^k(t, X_t^i, \PP_{\bm{X_t}}, Y_t^i, \PP_{\bm{Y_t}})=-\partial_x\mathcal{H}^k(t, X_t^i, \PP_{\bm{X_t}}, Y_t^i,\hat{\alpha}^k(t, X_t^i, \PP_{\bm{X_t}}, Y_t^i))$ where $\mathcal{H}$ represents the Hamiltonian of the problem and there is no need for dependence on the law $\PP_{\bm{Y_t}}$. Specifically, $\mathcal{H}^k(t, x, \mu, y, \alpha) = \langle b^k(t, x, \mu, \alpha), y \rangle + f^k(t, x, \mu, \alpha)$ and the functions $\hat{\alpha}^k$ are the optimizers of the $\mathcal{H}^k$. 
\end{remark}

This is a rich class of problems for which existence and uniqueness results are well-developed. See \cite{carmona2013probabilistic} for a set of assumptions that results in such an MV-FBSDE, for a solution exists and is unique. Furthermore, this class includes MFGs that involve endogenous price formation, such as the MFG problem discussed in \cite{shrivats2020mean}, \cite{fujii2020mean}, and \cite{gomes2018mean}. In \Cref{sec:example}, we discuss an MFG model with endogenous price formation pertaining to REC markets. It is these cases of endogenous price formation that can give rise to a representation where there is dependence of the drift coefficients on the law of the adjoint process.

\subsection{Principal's problem in a generic PA-MFG} \label{sec:general_principal_problem}

The principal aims to determine the terminal payment structure $g(X_T)$  that optimizes their own criterion. The principal derives a benefit from the output process and compensates the agent for said output at the terminal time $T$.

We restrict ourselves to payment structures such that the set $\mathcal{M}(g):= \{(\alpha, \bm{\mu}): (\alpha, \bm{\mu}) \text{ solves } \eqref{eq:MFG_optimal}-\eqref{eq:MFG_consistency}\}$ is non-empty. That is, the principal chooses over a set of functions such that the agents' MFG is well-posed. Furthermore, the principal ensures the agents' reservation cost is met (in expectation), across all sub-populations. 

We denote the set $\Xi$ as the set of admissible terminal payment structures for the principal to choose among. Let $\mathbf{V}^{\boldsymbol{\mu}}(g):=(\int V^{A, k}(g) d\mu_0^{(k)})_{k\in\mcK}$, $\Pi\geq 0$ be a full rank $s\times|\mcK|$ matrix (where the inequality is understood componentwise) and $\mathbf{R}_0$ be a $s$-dimensional vector of reservation costs. In particular, we say
\begin{equation}
    \Xi := \left\{g: \RR^d \rightarrow \RR \,\bigg\lvert \, \EE[g(X^{(k)}_T)^2] < \infty, \, g \text{ convex}, \, \mathcal{M}(g) \neq \emptyset,\, \Pi\mathbf{V}^{\boldsymbol{\mu}^\star}(g)\leq \mathbf{R}_0 \, \forall k \in \mcK\right\}. \label{eq:admissible_payments}
\end{equation}
If $\Pi=\mathrm{Id}$, then the constraint corresponds to a classic upper bound on the reservation costs for all agents. If $\Pi=\boldsymbol{\pi}\in\Delta_{|\mcK|}$ (where $\Delta_\cdot$ is the unit simplex), then the constraint corresponds to a restriction on the (weighted) average cost across sub-populations.

If $\mathcal{M}(g)$ is not a singleton, we assume the agent chooses the one that is best for the principal. As a result, the principal's problem may be written as the solution to\seb{Is there a reason not to remove the $-$ signs and change inf to sup? Also, is there a reason to have $g$ linear and not in the utility function: $U\left(X_T^{(k)} + \lambda g(X_T^{(k)})\right)$?}
\notesteven{I believe we just had it set up this way initially. We can change it to sup. As for the utility function, it might affect slightly the lemma later on, but we can also treat different cases.}
\seb{OK. noted -- indeed the Lemma would need modification for non-linear utility on the total value to the principal.}
\begin{equation}
    V^{P} := \inf_{g \in \Xi} \EE\left[\sum_{k \in \mcK}\pi_k\left(- U\left(X_T^{(k)}\right) - \lambda g(X_T^{(k)})\right)\right], \label{eq:principal_value}
\end{equation}
where $U$ is a concave and non-decreasing utility function, and $\lambda > 0$ is an exogenous principal parameter indicating their willingness to trade off between the output process and the contract they pay. Here, $\pi_k$ represents the fraction of the total population attributable to sub-population $k$ in the mean-field limit. The objective is to find $V^P$, the corresponding function $g$, and use them to characterize the resulting Nash equilibrium among the agents. 

In general, the principal's problem is challenging to solve analytically, even with the simplifications of uncontrolled volatility and restricting to Markov contract structures. The general approach to solving variations of these problems analytically is outlined by \cite{elie2019tale}, which was inspired by the work of \cite{cvitanic2018dynamic} and \cite{sannikov2008continuous}. Broadly speaking, the idea is to equate the solution of the MFG with the solution of a BSDE, and use this BSDE to define a set of coupled McKean-Vlasov SDEs, which describe (among other things) the dynamics of the agents' optimal value process. From here, the principal's problem can be turned into a MFC problem, with the agents' MFG solution embedded into the problem definition. This can either be solved using PDEs (as in \cite{elie2019tale}) or using the probabilistic approach (as in \cite{carmona2015forward}). In both cases, however, analytic solutions are impossible in all but the simplest of situations. Classical numerical approaches exist, but as always, one faces the Curse of Dimensionality. Instead, we take a DL approach and present an optimization algorithm that can more flexibly solve these type of problems.

\section{Numerical Algorithm} \label{sec:algorithm}

\subsection{Approximating the Principal's Objective}

We begin with some simplifying assumptions on $g$. In particular, we restrict to functions that are continuous, monotone, convex, and take a real valued argument. This is a classic setting for many MFGs. Convexity is often required to derive the form of the MV-FBSDE in Assumption \ref{ass:MVFBSDE} and monotonicity is natural as it relates to reward/cost structures that are monotone in the terminal state. It is known that such functions $g$ can be approximated uniformly on compact sets by piecewise linear functions (see for instance \cite{marsden1972uniform}).
Moreover, a convex piecewise linear function can be constructed through a positive linear combination of call and put payoffs. We use this insight to formulate the following.
\begin{assumption}\label{ass:g.function}
The function $g:\mathbb{R}\to\mathbb{R}$ is continuous and can be written as
\begin{equation}
    g(x)=\phi_0+\sum_{j=1}^{\bar{N}} w_j(x-R_j)^+,
        \quad 
        \text{or} 
        \quad
    g(x)=\phi_0+\sum_{j=1}^{\bar{N}} w_j(R_j-x)^+, \label{eqn:g.put.payoff}
\end{equation}
for weights $\phi_0\in\mathbb{R},w_j\in\mathbb{R}_+$ and knot points $R_j\in\mathbb{R}$, $j=1,...,{\bar{N}}$.
\end{assumption}

\begin{remark}
For uniqueness and existence results, we often also require $g$ to be differentiable with a Lipschitz continuous derivative.
However, since this holds in the formulation above, except at a finite number of points, and we use a simulation approach the lack of differentiability at these points does not affect the implementation. Nevertheless, while we do not treat it explicitly here, it is straightforward to regularize the piecewise linear approximations to produce smooth versions of them.
\end{remark}

Without loss of generality we treat the first case in \eqref{eqn:g.put.payoff}. For a fixed set of knot points the principal's objective function now becomes
\begin{equation*}
 \EE\left[\sum_{k \in \mcK}\pi_k\left(- U(X_T^{\star,(k)}) - \lambda\phi_0- \lambda\sum_{j=1}^{\bar{N}} w_j(R_j-X_T^{\star,(k)})^+\right)\right],
\end{equation*}
where the optimally controlled state processes are associated with the mean-field equilibrium induced by the principal's input weights $\mathbf{w}$.  The following simple lemma will prove useful as it allows us to eliminate one of the problem constraints.

\begin{lemma}\label{lem:equiv.rep}
Under Assumption \ref{ass:g.function} the principal's problem can be equivalently expressed as
\begin{equation}
    \inf_{\mathbf{w} \in \hat{\Xi}}\left( \lambda\max_{i\in[s]}\frac{\Pi^\top_i\mathbf{V}^{\boldsymbol{\mu}^\star}(\hat{g})-R_{0,i}}{\Pi_i^\top\mathbf{1}} + \EE\left[\sum_{k \in \mcK}\pi_k\left(- U(X_T^{\star,(k)})-\lambda\hat{g}(X_T^{\star,(k)})\right)\right]\right)
\end{equation}
where $X_T^{\star,(k)}$ is the optimally controlled process of a representative agent from population $k\in\mcK$ in equilibrium,
\begin{equation}
    \hat{g}(x):=\sum_{j=1}^{\bar{N}} w_j(R_j-x)^+,
\end{equation}
and
\begin{equation}
    \hat{\Xi}:=\left\{\mathbf{w}\in\RR^{\bar{N}}\bigg\lvert \EE[\hat{g}(X_T^{\star,(k)})^2] < \infty, \, \mathcal{M}(\hat{g}) \neq \emptyset, \ k\in\mcK\right\}.
\end{equation}
Here $\Pi_i$ is the vector corresponding to the $i$th row of the full rank matrix $\Pi\in\mathbb{R}_+^{s\times|\mathcal{K}|}$, $[s]:=\{1,...,s\}$, $R_{0,i}$ is the $i$th component of $\mathbf{R}_0$, and the agents' optimal controls for $\hat{g}$ and $g$ coincide.
\end{lemma}

\begin{proof}
Consider in the definition of $\Xi$ the constraint
\[\mathbf{V}^{\boldsymbol{\mu}^\star}(g) \leq \mathbf{R}_0.\]
Note that the agents' optimal control is invariant to the choice of $\phi_0$. Using this observation we may pull $\phi_0$ out from the agents' objective. Re-arranging yields that
\[\phi_0\Pi\mathbf{1}\leq \mathbf{R}_0-\Pi\mathbf{V}^{\boldsymbol{\mu}^\star}(\hat{g}).\]
From this and the positivity of $\Pi$ we can conclude
\[\phi_0\Pi\mathbf{1}\leq \mathbf{R}_0-\Pi\mathbf{V}^{\boldsymbol{\mu}^\star}(\hat{g}) \ \ \  \iff \ \ \  \phi_0 \leq \min_{i\in[s]}\frac{R_{0,i}-\Pi^\top_k\mathbf{V}^{\boldsymbol{\mu}^\star}(\hat{g})}{\Pi_k^\top\mathbf{1}}.\]
We can divide by $\Pi_k^\top\mathbf{1}$ since $\Pi$ is full rank and thus has no rows that are identically $0$.
Furthermore, if $\phi_0$ is optimal then equality holds in the above. If not, the principal can always improve their objective by increasing $\phi_0$ to this upper bound since it has no impact on the agents' controls - a contradiction. On the other hand, if we restrict to the above equality, minimizing the principal's objective over $\mathbf{w}\in \hat{\Xi}$ determines the minimizing value of $\phi_0$. The claim in the lemma then follows by substitution.
\end{proof}

We also enforce the following standard assumptions that simplifies the optimization.
\begin{assumption} \label{ass:set.hat.Xi} The following conditions hold:
\begin{enumerate}
    \item For any $\alpha\in\mathcal{A}^k$, $\phi_0\in\mathbb{R}$, and $\mathbf{w}=(w_1,\dots,w_{\bar{N}})^\top \in\mathbb{R}^{\bar{N}}_+$, the drift coefficient $b^k(t,x,p,\alpha)$ and diffusion coefficient $\sigma^k(t)$ are such that $\EE\left[g(X^{(k)}_T)^2\right]<\infty$, for all $k\in\mcK$.
    \item There exists a $E\subseteq\mathbb{R}^{\bar{N}}$ such that $\mathbf{w}\in E$ if and only if $\mathcal{M}(g)\not=\emptyset$.
    \item There exists an almost-everywhere differentiable mapping $\psi$ such that $\psi[\mathbb{R}^{\bar{N}}]=E$.
\end{enumerate}
\end{assumption}
This assumption allows us to replace the constraint $\mathbf{w}\in \hat{\Xi}$ with $\mathbf{w}\in E$ which is straightforward to implement numerically. In particular, if Assumption \ref{ass:set.hat.Xi}(3) holds we can optimize directly on $\mathbb{R}^{\bar{N}}$ by defining $\mathbf{w}:=\psi(\mathbf{u})$ for $\mathbf{u}\in\mathbb{R}^{\bar{N}}$.

\subsection{Approximating the MV-FBSDE}

To solve the PA problem we first need a mechanism for solving the agent problem through the MV-FBSDE \eqref{eq:MV-FBSDE_fwd}-\eqref{eq:MV-FBSDE_bwd} for different weights $\mathbf{w}$ specified by the principal. We tackle this by discretizing said MV-FBSDE, and parameterizing the co-adjoint process and initial condition by neural nets akin to the deep BSDE/FBSDE techniques employed in \cite{han2020convergencebsdes,han2018solving,germain2019numerical,carmona2021convergence}. We note that a new addition in our setting is the dependence on the law of the adjoint process. We briefly introduce a new notation, whereby $[\cdot]^{l, (k)}$ represents the $l$-th sampled path of $[\cdot]$ for a representative agent from sub-population $k$ in our discretization. We use $l$ to avoid any possible confusion with agent indexing as used in prior sections, and do not mix agent indexing and representative indexing to maintain clarity. 

Letting $\mathbb{T}=\{t_0,...,t_M\}$ be a discrete set of points with $t_0=0$ and $t_M=T$ we define an ensemble of neural nets $Y^{\theta_{0}^{(k)}}$, $(Z^{\theta_{m}^{(k)}})_{m=1}^M$ and introduce the time discretized version of \eqref{eq:MV-FBSDE_fwd}-\eqref{eq:MV-FBSDE_bwd}:
\begin{align}
X_{t_m}^{l,(k)} 
&=
X_{t_{m-1}}^{l,(k)}+\varphi^k(t,X_{t_{m-1}}^{l,(k)},\hat{\mathbb{P}}_{\mathbf{X_{t_{m-1}}}},Y_{t_{m-1}}^{l,(k)},\hat{\mathbb{P}}_{\mathbf{Y_{t_{m-1}}}})\Delta t +\sigma^k(t_{m-1})\Delta W^i_{t_m}
\label{eqn:discretized.FBSDE.X}
\\
X_0^{l,(k)} &=\xi^l
\label{eqn:discretized.FBSDE.X0}
\\
Y_{t_m}^{l,(k)}
&=
Y_{t_{m-1}}^{l,(k)}+\rho^k(t,X_{t_{m-1}}^{l,(k)},\hat{\mathbb{P}}_{\mathbf{X_{t_{m-1}}}},Y_{t_{m-1}}^{l,(k)},\hat{\mathbb{P}}_{\mathbf{Y_{t_{m-1}}}})\Delta t
+ Z^{\theta_{m}^{(k)}}(X_{t_{m-1}}^{l,(k)},Y_{t_{m-1}}^{l,(k)})\Delta W^l_{t_m}
\label{eqn:discretized.FBSDE.Y}
\\
Y_{0}^{l,(k)}
&=
Y^{\theta_{0}^{(k)}}(X_0^{l,(k)})
\label{eqn:discretized.FBSDE.Y0}
\end{align}
for $\Delta t=T/M$, $\Delta W_{t_m}^i=W^l_{t_m}-W^l_{t_{m-1}}$ and $m=1,...,M$. Here $\boldsymbol{\theta}:=(\theta_{m}^{(k)})_{m\in [N],k\in\mathcal{K}}$ are the network parameters, and $\hat{\mathbb{P}}_\cdot$ denotes the empirical distribution over the samples.
To solve the original MV-FBSDE we have also sampled finitely many paths ($\mathcal{N}_k$ paths) for each sub-population $k\in\mathcal{K}$. For compactness of notation we write $\mathbf{X}^{(k)}_t:=(X^{1,(k)}_t,...,X_t^{\mathcal{N}_k,(k)})^\top,\mathbf{Y}^{(k)}_t:=(Y^{1,(k)}_t,...,Y_t^{\mathcal{N}_k,(k)})^\top$ for $k\in\mathcal{K}$ and $t\in\mathbb{T}$
to represent the samples of each population. When the superscript $k$ is omitted, we refer to the samples across all sub-populations. 

The intuition of the algorithm is that we optimize over the parameters $\boldsymbol{\theta}$ in order to match the terminal condition for $Y_T^{l,(k)}$. To this end, we define the MV-FBSDE loss function on a set of sample  paths as
\begin{equation}\label{eqn:FBSDE.loss}
\mathcal{L}_F(\boldsymbol{\theta})=\frac{1}{|\mathcal{K}|}\sum_{k\in\mathcal{K}}\frac{1}{\mathcal{N}_k}||\mathbf{Y}^{(k)}_{t_{M}}-\partial_xg(\mathbf{X}^{(k)}_{t_{M}})||_2^2.
\end{equation}
We say that we have an approximate solution to the MV-FBSDE when the loss function has converged within a specified tolerance to $0$. This loss function implicitly depends on the weights $\mathbf{w}$ through the definition of $g$.

\subsection{Learning the Principal Objective Function}

To tackle the principal problem we discretize the running costs using a Trapezoidal rule to optimize the objective with respect to the weights $\mathbf{w}$.
Here, for a fixed set of weights, we simulate $\mathcal{N}_k$ sample paths for each sub-population $k\in\mcK$ to approximate the expectation. We also assume a uniform grid $\mathbb{T}$ so that we may write\seb{why use $\star$ below?}\notesteven{I think originally it was because we want to use the approximate optimally controlled process, but I think in this case this is confusing. In practice we just use the learned process so I will remove the stars and discuss with Arvind \& Yichao to see what a better way to convey this is. - UPDATE: we changed this to not include a star.}
\seb{I normally think of $\star$ as describing something optimal... maybe a different notation like $X^\diamond$ as you intend to use the approximate optimal solution here?}
\begin{align*}
\mathrm{Running \ Costs}_l^{(k)}&\approx\Delta t\sum_{m=1}^{M-1}f^k(t, X^{l,(k)}_{t_m}, \hat{\mathbb{P}}_{\mathbf{X}_{t_m}}, \hat{\alpha}_{t_m}^{l})+\frac{\Delta t}{2}\sum_{m\in\{0,M\}}f^k(t, X^{l,(k)}_{t_m}, \hat{\mathbb{P}}_{\mathbf{X}_{t_m}}, \hat{\alpha}_{t_m}^{l}).
\end{align*}
Our principal loss on a sample is given by\seb{not sure if the definition of $[s]$ has been given}\notesteven{It is above (2.12) and now also in Lemma 3.3. It corresponds to the number of rows of $\Pi$, but maybe there is a better way to reference the dimension here... }\yichao{probably $\max_{1\leq i\leq s}$?}
\begin{equation}\label{eqn:sample.principal.loss}
    \hat{\mathcal{L}}_P(\mathbf{w})= \lambda\max_{i\in[s]}\frac{\Pi^\top_i\hat{\mathbf{V}}^{\hat{\mathbb{P}}}(\hat{g})-R_{0,i}}{\Pi_i^\top\mathbf{1}} + \sum_{k \in \mcK}\frac{\pi_k}{\mathcal{N}_k}\sum_{l=1}^{\mathcal{N}_k}\left(- U(X_T^{l,(k)})-\lambda\hat{g}(X_T^{l,(k)})\right),
\end{equation}
where the $k^{th}$ component of $\hat{\mathbf{V}}^{\hat{\mathbb{P}}}(\hat{g})$ is given by:
\[
\hat{\mathbf{V}}_k^{\hat{\mathbb{P}}}(\hat{g}):=\frac{1}{\mathcal{N}_k}\sum_{l=1}^{\mathcal{N}_k}\left(\mathrm{Running \ Costs}_l^{(k)}+ \hat{g}(X^{l,(k)}_{t_M})\right).
\]
We use this to estimate the true principal loss at $\mathbf{w}$.

\begin{wrapfigure}{r}{0.6\textwidth}
\vspace*{-0.5em}
\begin{minipage}{0.6\textwidth}
\begin{algorithm}[H]
\scriptsize
\SetAlgoLined
 Initialize forward network parameters $\mathbf{\theta}$.  Initialize principal parameters $\Upsilon$.  Initialize memory buffer $\mathfrak{M}$. Initialize $\mathbf{u}^{(0)}$
 \\
 $N_O:=$ max \# of outer steps, $N_S\geq \bar{N}:=$ local sample size, $N_F:=$ \# MV-FBSDE  steps, $N_A:=$ \# loss approximation steps, $N_P:=$ \# principal loss  steps,  $(\mathcal{N}_k)_{k\in\mcK}:=$ MV-FBSDE sample size, $N_B:=$ \# loss approximation batch size\;
 Define weight error tolerance $TOL$ and MV-FBSDE error tolerance $TOL_F$\;
 Define update rule for sampling radius $\epsilon$, network parameters $\boldsymbol{\theta},\Upsilon$ and values $\mathbf{u}^{(\cdot)}$\;
 \ForEach{$n$ \ $\mathrm{in}$ \ $[N_O]$}{
    Initialize local samples $\mathcal{U}:=(\mathbf{u}_i^{(j)})_{i\in[N_S]}\subset B_{\epsilon}(\mathbf{u}^{(j)})$\;
    \ForEach{$\mathbf{u}'\in \mathcal{U}$}{
      \ForEach{$k \ \mathrm{in} \ [N_F]$}{
        Sample paths of the discretized MV-FBSDE \eqref{eqn:discretized.FBSDE.X}-\eqref{eqn:discretized.FBSDE.Y0} by generating initial conditions, Brownian noise, and passing through the ensemble network $\boldsymbol{\theta}$ using the sample sizes $(\mathcal{N}_k)_{k\in\mcK}$\;
        Compute the MV-FBSDE loss $\mathcal{L}_F(\boldsymbol{\theta})$ (see \eqref{eqn:FBSDE.loss})\;
        Update $\boldsymbol{\theta}$\;
        \If{$\mathcal{L}_F(\boldsymbol{\theta})<TOL_F$}{
        break\;
    }
    }
    Estimate the principal loss via $\hat{\mathcal{L}}_P(\psi(\mathbf{u}'))$ (see \eqref{eqn:sample.principal.loss})\; 
    Store $(\mathbf{u}',\hat{\mathcal{L}}_P(\psi(\mathbf{u}')))$ in $\mathfrak{M}$\;
    }
    \ForEach{$k \ \mathrm{in} \ [N_A]$}{Sample from $\mathfrak{M}$ a batch $\mathbf{U}:=(\mathbf{u}_m)_{m\in[N_B]}$ and associated principal losses $\hat{\mathcal{L}}_P(\psi(\mathbf{U})):=(\hat{\mathcal{L}}_P(\psi(\mathbf{u}_m)))_{m\in[N_B]}$\; 
    Compute the MSE loss $||\mathcal{L}^{\Upsilon}_P(\psi(\mathbf{U}))-\hat{\mathcal{L}}_P(\psi(\mathbf{U}))||_2^2$\;
    Update $\Upsilon$\;
    }\ForEach{$k$ \ $\mathrm{in}$ \ $[N_P]$}{
        Compute $\mathcal{L}^{\Upsilon}_P(\psi(\mathbf{u}^{(j)}))$\;
        Update $\mathbf{u}^{(j)}\mapsto\mathbf{u}^{(j+1)}$\;
        Increment $j\mapsto j+1$\;
    }
    Update $\epsilon$\;
    \If{$||\mathbf{u}^{(j)}-\mathbf{u}^{(j-N_P)}||_2<TOL$}{
        break\;
    }
    }
    
 \caption{Principal-Agent Optimization}
 \label{alg:PA.optimization}
\end{algorithm}
\end{minipage}
\vspace*{-2em}
\end{wrapfigure}
Given this estimate of the principal's loss we aim to optimize it over the weights, which requires gradients. To compute an approximate gradient, we  estimate the principal loss function over the weights locally. In particular, we sample $N$ points uniformly in an $\epsilon$-ball about the current estimate of the weights. The principal loss on each of these samples may be estimated by solving the 
MV-FBSDE. With these estimates of the loss within an $\epsilon$-ball, we use a second neural net,
$\mathcal{L}_P^\Upsilon(\mathbf{w})$ with parameters $\Upsilon$ that fits the target  $\hat{\mathcal{L}}_P(\mathbf{w})$. From this estimate, we may obtain estimates of the gradients and perform an approximate gradient update. For simplicity, the optimization  over weights proceeds on $\mathbb{R}^{\bar{N}}$ by using  $\psi$ from Assumption \ref{ass:set.hat.Xi}. More specifically, we take gradient of $\mathcal{L}_P^\Upsilon(\psi(\mathbf{u}))$ with respect to the input $\mathbf{u}$  and perform a  gradient step  to determine the next value of $\mathbf{u}$. This, in turn, defines updated values of $\mathbf{w}$.

\subsection{Algorithm}

We now have all the components necessary to describe the algorithm for solving the PA problem fully. Our approach may be viewed as having two main components: an inner optimization (for the agents) and an outer optimization (for the principal). The approach works at a high level by (i) initializing $\mathbf{u}^{(0)}\in\mathbb{R}^{\bar{N}}$, (ii) performing $N_F$ optimization steps for the inner agent problem, (iii) fixing the agents' controls, estimate the principal's loss function, and (iv) performing $N_P$ gradient updates of the outer problem. This is repeated until the Principal's weights and the MV-FBSDE loss are  within a given tolerance. The details are provided in Algorithm \ref{alg:PA.optimization}.

\section{Example} \label{sec:example}

As a motivating example of a real-world problem where our methodology could be applied, we consider Renewable Energy Certificate (REC) markets. The optimal behavior of agents in REC markets is covered in detail from a MFG perspective in \cite{shrivats2020mean}, however, we extend their model in two significant ways: first by extending the problem to include the regulator as a principal, and second by expanding the state and control space of the model for a more realistic formulation.

\subsection{The Basics of REC Markets}

REC markets belong to the class of so-called market-based emissions regulation policies, the most well known of which are carbon cap-and-trade (C\&T) markets.

REC markets are a closely related alternative to C\&T markets. In REC markets, a regulator sets a floor on the amount of energy generated from renewable sources for each firm (based on a percentage of their total energy generation), and provides certificates for each MWh of energy produced via these means. This is also known as a Renewable Portfolio Standard (RPS), and typically apply to private Load Serving Entities (LSEs), also known as electricity suppliers. To ensure compliance, each firm must surrender certificates totaling the floor at the end of a compliance period, with a monetary penalty paid for each lacking certificate. The certificates are traded assets, allowing regulated LSEs to make a choice about whether to produce electricity from renewable means themselves, or purchase the certificates on the market (or a mix of both). In either case, whether purchasing or producing clean, RECs induce the production of clean energy, as all purchased RECs must have been generated through clean generation means.

In practice, these systems regulate multiple consecutive and disjoint compliance periods, which are linked together through a mechanism called \textit{banking}, where unused allowances in period-$n$ can be carried over to period-$(n+1)$. However, we consider a simpler single-period framework in this example.


\subsection{Modeling of REC markets}
We discuss the single-period framework for REC markets with the following rules. The market governs a compliance period from $[0,T]$ denoted as $\mfT$. A firm obtains RECs in each period, with their terminal RECs denoted by $X_T$. At time $T$ a firm must pay (or possibly is paid) $C(X_T)$, which represents a generic penalty chosen by the regulator, which is made known to the firms. We assume there are no costs after time $T$.

Firms receive RECs by generating electricity through a particular energy source (depending on the market, this could be solar, nuclear, etc.). One REC typically corresponds to one MWh of electricity produced via the target energy source. A firm may also purchase or sell RECs on the market. After $T$, all firms forfeit any remaining RECs. $T$ can be thought of as `the end of the world' -- there are no costs associated with any time after this.

\subsubsection{The Agent Problem in REC Markets}
Agents belonging to sub-population $k\in\mcK$ are assumed to have a baseline REC generation rate $h_t^k$ at which they generate with zero marginal cost. As discussed in Section \ref{sec:general_agent_problem}, agents preferences are assumed to be identical within sub-populations, but are distinct across them. All agents are minor agents; that is, they have no market impact individually. 

Agents have the ability to further increase $h_t^k$ through their rate of expansion activities. In essence, this models building more REC generation capacity. In the context of a Solar REC market, this can be thought of as installing more solar panels. This control is denoted by $(\alpha_t^i)_{t \in \mfT}$, with the units being REC capacity added per unit time. 

The agents also have the ability to rent short-term REC generation capacity for an assumed quadratic cost, denoted by the control $(g_t^i)_{t \in \mfT}$. Once again, this represents a rate, this time of generation capacity rented per unit time.

Finally, agents may engage in the REC market and purchase/sell RECs as needed, with their purchasing rate in the market denoted by $(\Gamma_t^i)_{t \in \mfT}$. Generation expansion and rental must be positive, while trading rate may be positive or negative. 

We denote the collection of expansion rates, generation rental rates, and trading rates by $\bm{\alpha_t} := (\alpha_t^i)_{i \in \mfN}$, $\bm{g_t} := (g_t^i)_{i \in \mfN}$ and $\bm{\Gamma_t} := (\Gamma_t^i)_{i \in \mfN}$, respectively. The addition of expansion generation as a control differentiates this example from \cite{shrivats2020mean}. This model could be made even more realistic by incorporating a delay between the decision to expand and the increase to the baseline generation rate. 

The agents therefore must keep track of two state variables: their REC inventory, and their REC capacity. The relevant state variables for agent $i$ in sub-population $k$ are denoted $X^i$ and $C^i$ and they satisfy the SDEs
\begin{align}
    dX_t^i &= (h_t^k + C_t^i + \Gamma_t^i) dt + \sigma^k dW_t^i, \qquad &X_0^i = \xi^i \sim F(\Theta^k) \\
    dC_t^i &= \alpha_t^i dt + 0\, dB_t^i, \qquad &C_0^i = 0.
\end{align}
We include the diffusive term in $C_t^i$ even though its coefficient is 0 to make clear that in principle it may be stochastic as well -- e.g., a firm pays for and plans expansion capacity but the capacity is not delivered.

As in \Cref{sec:general_agent_problem}, let $\mu_t^{(k)}$ represent the mean-field distribution of states. That is for all $t \in \mfT$, $\mu_t^{(k)} \in \mathcal{P}(\RR^2)$.  Similarly, $\bm{\mu} = (\{\mu_t^{(k)}\}_{k \in \mcK})_{t \in \mfT}$. Moreover,  $\mu_t^{(k), X}$ and $\mu_t^{(k), C}$ denotes the marginal mean-field distributions of the denoted state. \Cref{IntialStateAss} and \Cref{ass: MinorContrAction} are assumed to hold.

The agents seek to profit maximize, and in particular, they aim to optimize the following cost functional:
\begin{align}
    J^{A, i}(\alpha, g, \Gamma; \bm{\mu}) = \EE\biggl[&\int_0^T\frac{\zeta^k}{2}(g_u^i)^2+\frac{\gamma^k}{2}(\Gamma^i_u)^2 + \frac{\beta^k}{2} (\alpha_u^i)^2+S^{\bm{\mu}}_u\Gamma^i_udu 
    + \phi_0+\sum_{j=1}^N w_j(R_j-X_T^{i})^+ \; \rvert \;\Gg_0^i \biggr], 
    \label{eq:agent_performance_REC}
\end{align}
with the form of the terminal condition given in  \Cref{ass:g.function}.
$S^{\bm{\mu}}$ represents the equilibrium REC price. This cost functional can be thought of as the sum of various costs. The controls all have quadratic costs associated with them, with the latter two terms representing the cost (or profits) from trading and the non-compliance penalty imposed by the principal, respectively. $\zeta^k, \gamma^k, \beta^k$ are scalar cost parameters which are identical for firms within a sub-population, and possibly different across them. 

The REC price is obtained endogenously through a market clearing condition, as in \cite{shrivats2020mean, fujii2020mean, gomes2018mean}. The clearing conditions imply that
\begin{equation}
    \lim_{N \rightarrow \infty} \tfrac{1}{N}\sum_{i\in\mfN}\Gamma_t^i = 0,
\end{equation}
 and this condition induces an expression for the REC Price $S^{\bm{\mu}}$ provided in \eqref{eqn:price.path.sol}.
The clearing condition amounts to the average trading rate (across agents) vanishing at all times in the limit of large number of players. Note that in the individual agents' cost functional \eqref{eq:agent_performance_REC}, the mean-field distribution arises only through the REC price. 

As in Section \Cref{sec:general_agent_problem}, the agents seek to find a Nash equilibrium as defined in \eqref{eq:MFG_optimal} and \eqref{eq:MFG_consistency}. One can show (following the steps of \cite{shrivats2020mean}, or the probabilistic approach espoused in \cite{carmona2013probabilistic}) that the solution for agent $i$ in sub-population $k$ for all $i \in \mfN_k$, $k \in \mcK$, can be found through the solution to the following MV-FBSDE
\begin{align}
    dX_t^i &= (h_t^k - \left(\tfrac{1}{\zeta^k}+\tfrac{1}{\gamma^k}\right)Y_t^{i, X} - \tfrac{1}{\gamma^k} S_t^{\bm{\mu}} + C_t^i) dt + \sigma^k dW_t^i, \; &&X_0^i = \xi^i,
    \label{eq:FBSDE_fwd1}
    \\
    dC_t^i &= -\tfrac{1}{\beta^k} Y_t^{i, C} dt + 0 dB_t^i, \; &&C_0^i = 0,
    \label{eq:FBSDE_fwd2}
    \\
    dY_t^{i, X} &= Z_t^{i, X} dW_t^i, \; &&Y_T^{i, X} = \partial_x g(X_T^i),
    \label{eq:FBSDE_bwd1}
    \\
    dY_t^{i, C} &= - Y_t^{i, X} dt + Z_t^{i, C} dB_t^i, \qquad &&Y_T^{i, C} = 0,
    \label{eq:FBSDE_bwd2}
\end{align}
with the optimal controls given by
\begin{align}
    \alpha_t^{i, \star} &= -\tfrac{Y_t^{i, C}}{\beta^k},
    \\
    g_t^{i, \star} &= -\tfrac{Y_t^{i, X}}{\zeta^k}, 
    \\
    \Gamma_t^{i, \star} &= -\tfrac{Y_t^{i, X} - S_t^{\bm{\mu}}}{\gamma^k}, \quad \text{and}
    \\
    S_t^{\bm{\mu}} &= \frac{-1}{\sum_{k \in \mcK}\frac{\pi_k}{\gamma_k}}\sum_{k\in\mathcal{K}}\frac{\pi_k}{\gamma^k}\mathbb{E}\left[Y_t^{(k),X}\right].
    \label{eqn:price.path.sol}
\end{align}
Therefore, the agents' problem fits into the formulation described in \Cref{sec:general_agent_problem}. We now discuss the principal's problem. 

\subsubsection{The Principal's Problem in REC Markets}

The principal aims to solve the optimization problem:
\begin{align*}
    \inf _{\phi_0\in\mathbb{R},\mathbf{w}\in\mathbb{R}_{+}^{\bar{N}}} J^{P}\left(\phi_0,\mathbf{w}\right)=\inf _{\phi_0\in\mathbb{R},\mathbf{w}\in\mathbb{R}_{+}^{\bar{N}}} \mathbb{E}\left[\sum_{k \in \mathcal{K}} \pi_{k}\left(-\lambda \left(\phi_{0}+\sum_{j=1}^{m} w_{j} (R_{j}-X_T^{(k),\star})^+\right)-X_{T}^{(k),\star}\right)\right],
\end{align*}
where $X_t^{(k),\star}$ represents the optimal controlled REC inventory for a representative agent in sub-population $k$, and $\mathbf{w}=(w_1,\dots,w_{\bar{N}})^\top$ subject to the constraint
\[
\boldsymbol{\pi}^\top\mathbf{V}^{\boldsymbol{\mu}^\star}(g)\leq R_0,
\]
where $\boldsymbol{\pi}\in\Delta_{|\mcK|}$ is the weight vector whose entries correspond to the fraction of the market in the mean-field that is attributable the sub-populations $k\in\mcK$. When expanded this may be written
\[
\sum_{k=1}^K\pi_k  \int J^{A, (k)}(\alpha^{(k), \star},g^{(k),\star},\Gamma^{(k),\star}; \bm{\mu^\star})d\mu_0^{(k),\star}\leq R_0,
\]
where $(\alpha^{(k), \star}, g^{(k),\star},\Gamma^{(k),\star})$ represent the collection of optimal controls across representative agents from each group $k\in \mathcal K$ and $J^{A, (k)}$ denotes their cost functional. $\bm{\mu}^\star$ represents the induced mean-field flow from the optimal controls. The constant $R_0$ represents the reservation cost.

It is possible (via Lemma \ref{lem:equiv.rep}) to see that the principal's problem is equivalent to the modified problem:
\begin{equation*}
\inf _{\mathbf{w}\in\mathbb{R}_{+}^N}\mathbb{E}\left[\sum_{k \in \mathcal{K}} \pi_{k}\left(-\lambda\left(R_{0}-\mathrm{Running \ Costs}^{k,\star}\right)-X_{T}^{k,\star}\right) \cdot\right]
\end{equation*}
where $\mathrm{Running \ Costs}^{k,\star} = \int_0^T\frac{\zeta^k}{2}(g_u^{(k),\star})^2+\frac{\gamma^k}{2}(\Gamma^{(k), \star}_u)^2 + \frac{\beta^k}{2} (\alpha_u^{(k), \star})^2+S^{\bm{\mu}^\star}_u\Gamma^{(k), \star}_udu$. In this representation we have \[\phi_0=R_0-\sum_{k\in\mathcal{K}} \pi_j \mathbb{E}\left[\mathrm{Running \ Costs}^{k,\star}+\sum_{j=1}^Nw_j(R_j-X_T^{k,\star})^+\right].\]
Note that $X_T^{k,\star}$ and the running costs depend implicitly on $\mathbf{w}$ through the optimal controls.

Both the agent and principal's problems fit into the overall framework discussed in \Cref{sec:general_agent_problem} and \Cref{sec:general_principal_problem}, more specifically,  Algorithm \ref{alg:PA.optimization} may be applied to them. Moreover, this problem is much harder to solve using the methods described in \cite{shrivats2020mean}, with the increased state space posing a challenge to the semi-analytic iterative scheme proposed therein. Thus, the algorithm unlocks a problem that is otherwise infeasibly difficult with standard numerical methods.

\section{Results} \label{sec:results}

We next proceed to implementing the algorithm on the PA REC problem and discuss the results.
We first consider a relatively simpler model of the REC market example discussed above, as a method to establish some evidence of the appropriateness of the solution output by Algorithm \ref{alg:PA.optimization}. Specifically, we first consider the scenario where the terminal cost $g(X_T)$ as in \eqref{eqn:g.put.payoff} has only a single knot point (i.e. $\bar{N}$ = 1). Upon validating our solution in this setting, we proceed to implement our algorithm on a similar synthetic example where $\bar{N} = 10$. This allows for a richer class of terminal penalty functions to be searched over from the principal's perspective, and results in a more interesting optimization problem. 
In this section, we use the parameters given in \Cref{tab:compliance.parameters} and \Cref{tab:model.parameters} throughout. The distribution of initial inventory $X_0^{(k)} \sim \xi^{(k)} = \mathcal{N}(v^k, \eta^k)$, for all agents in sub-population $k$, for all $k \in \mcK$. 

\begin{table}[ht]
\centering
\begin{tabular}[t]{cccccc}
\toprule\toprule
$\Delta_t$ & $T$ &$K$ & $\lambda$& $R_0$\\
\hline
$1/52$ & $1$ &$2$ & $6$ &$0$\\
\toprule\toprule\\
\end{tabular}
\caption{Compliance Parameters}
\label{tab:compliance.parameters}
\end{table}%
\begin{table}[ht]
\centering
\begin{tabular}[t]{cccccccccc}
\toprule\toprule
$\text{Sub-population}$ & $\pi_k$ & $h^k$ &$\sigma^k$ & $\zeta^k$ &$\gamma^k$&$v^k$&$\eta^k$&$\beta^k$\\
\hline
$k=1$ &$0.25$&$0.2$& $0.1$ & $1.75$ &$1.25$&$0.6$&$0.1$&1.0\\
\hline
$k=2$ &$0.75$&$0.5$& $0.15$ & $1.25$ &$1.75$&$0.2$&$0.1$&1.0\\
\toprule\toprule\\

\end{tabular}
\caption{Model Parameters}
\label{tab:model.parameters}
\end{table}%
These parameters are not calibrated to data, but chosen to provide stylized differences between the sub-populations that will be reflected in their optimal controls, as in \cite{shrivats2020mean} and \cite{shrivats2020optimal}. In particular, we note that a firm from population 1 is more likely to have a much greater initial allocation of RECs than one from population 2. Furthermore, they have greater rental costs, but a lower cost of trading RECs, as well as a significantly lower baseline REC generation capacity. These differences in parameters results in optimal controls that delineate sub-populations for their particular needs in this system.

\subsection{Single knot non-compliance penalty}

For the single knot model, we fix $g$ such that it has a single knot point at $R_1 = 0.9$. Assuming $R_0 = 0$, the terminal condition $\partial_xg$ of the MV-FBSDE becomes a step function 
\begin{align*}
    \begin{cases} 
    \partial_xg(x) = -w & x \leq 0.9,
    \\
    \partial_xg(x) = 0 & x > 0.9.
    \end{cases}
\end{align*}

\seb{In the left plot, it is not clear the weights really converged, could you run it with more iterations to show the convergence?}\notesteven{We can do that - currently the other optimization is running, but we will update this next.}We apply Algorithm \ref{alg:PA.optimization} in this case, and in Figure \Cref{fig:1.node.principal.loss.and.w.trajectory}  plot (left) the trajectory of $w$ across the optimization, and (right) the estimated principal loss across a set of candidate $w$'s. Specifically, we conducted a grid search for $w \in [0.05, 0.40]$ by applying our algorithm and saving the trained ensemble model for each $w$, which we denote by $\boldsymbol{\theta}_w$. For each of these models, we sampled $100$ new batches of Brownian noise and initial conditions, collecting the resulting means and standard errors of the principal losses. From these plots, we see that the optimal weight suggested by the algorithm quickly entered a band roughly between $0.18$ and $0.24$ and then eventually settles near $0.205$. Meanwhile, the curve of principal loss appears to be convex in $w$, with the minimum occurring right around this number, albeit with a relatively flat area surrounding it. 

\begin{figure}[ht]
\begin{center}
    \includegraphics[width=0.6\textwidth]{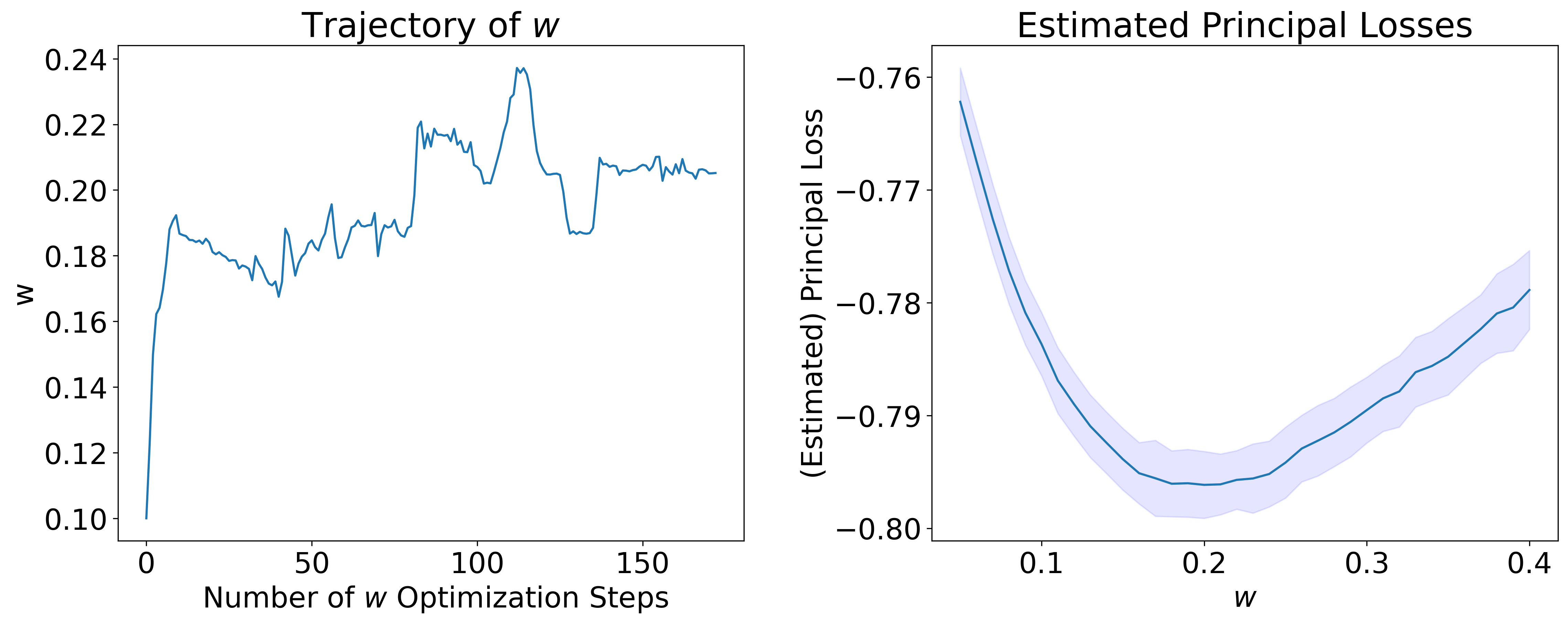}
\end{center}
\caption{Trajectory of $w$ during optimization (left) and estimated principal loss as a function of the weight $w$ (right) within the single knot test case ($R_1 = 0.9)$. Parameters as in \Cref{tab:compliance.parameters} and \Cref{tab:model.parameters}.}
\label{fig:1.node.principal.loss.and.w.trajectory}
\end{figure}

\subsection{Multiple knot non-compliance penalty}

The single knot example discussed above provides some level of evidence that our algorithm can properly identify the optimal non-compliance penalty of the principal. We now progress to a more involved example. In this subsection we fix $10$ nodes from $0.8$ to $1.16$ with an equal spacing of $0.04$ (see \Cref{tab:nodes.and.weights} for the precise knot points). With the multi-knot setting, we are able to optimize over a richer class of non-increasing convex terminal costs, and one that in fact encompasses the simple single knot setting.

In \Cref{fig:10.node.principal.loss.and.w.trajectory}, we plot the optimal weight trajectories throughout the optimization (left), as well as the principal loss trajectory (right). The principal loss converges quickly to approximately $-0.81$, while it takes longer for the weights $w_i$ themselves to converge. This indicates that the minimizer of the principal loss is located in a relatively flat region. It is worth noting that the converged $w$ has many of its coordinates close to zero. In fact, only the weights for the knots at $0.8, 0.84, 1.04$ and $1.08$ RECs are non-negligible (see \Cref{tab:nodes.and.weights}. 

\begin{figure}[ht]
\begin{center}
    \includegraphics[width=0.7\textwidth]{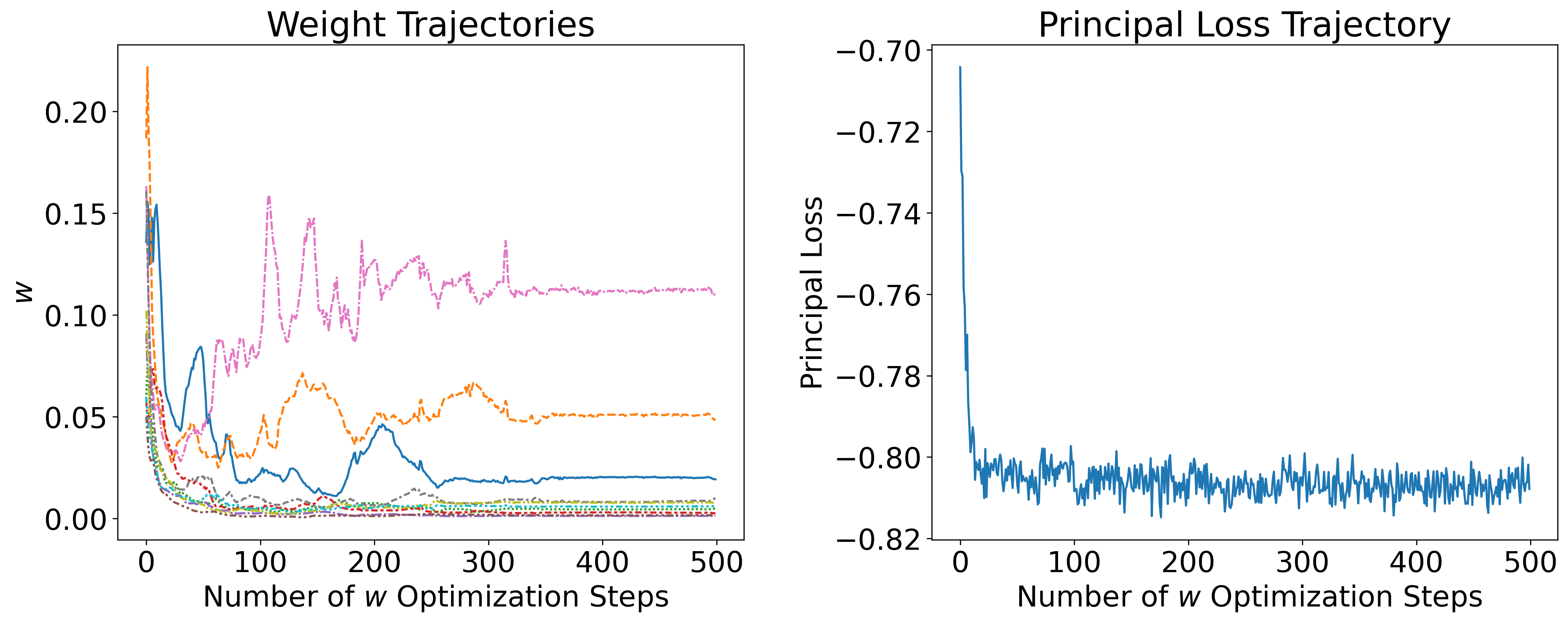}
\end{center}
\caption{Trajectory of optimal weights (left) and principal loss (right) within the multiple knot test case. Parameters as in \Cref{tab:compliance.parameters} and \Cref{tab:model.parameters}. Knots as in \Cref{tab:nodes.and.weights}.}
\label{fig:10.node.principal.loss.and.w.trajectory}
\end{figure}

\begin{table}[ht]
\centering
\begin{tabular}[t]{ccccccccccc}
\toprule\toprule
$\text{Node}$ & $0.8$ & $0.84$ & $0.88$ & $0.92$ & $0.96$ &  $1$.& $1.04$ & $1.08$ & $1.12$ & $1.16$ 
\\
$\text{Weight}$  & $0.0193$ & $0.0483$ & $0.0050$ & $0.0025$ & $0.0014$ & $0.0018$ & $0.1097$ & $0.0101$ & $0.0077$ & $0.0063$\\
\toprule\toprule\\
\end{tabular}
\caption{Knots and corresponding optimal weight in non-compliance penalty for multiple knot test case. Parameters as in \Cref{tab:compliance.parameters} and \Cref{tab:model.parameters}.}
\label{tab:nodes.and.weights}
\end{table}%

We use the output of the optimization to plot the optimal non-compliance penalty of the principal, and the distribution of induced non-compliance costs across the samples of representative agents from the differing sub-populations in \Cref{fig:10.node.penalty}. Note that the samples do not represent a finite set of agents competing in the same market and rather represent samples from the (approximate) mean-field distribution.

We observe that the non-compliance penalty $g$ is much steeper for $x<1.04$ than that for $x \geq 1.04$ , indicating a higher marginal penalty for insufficient inventory below $1.04$. Turning our attention to the right panel in \Cref{fig:10.node.penalty}, we see that a representative firm from population $2$ (blue) pays more on average in non-compliance penalties than one from population $1$ (red).\seb{perhaps has more bins in this range.. perhaps 25 or 31? it may reveal more interesting structure.} 
\begin{figure}[ht]
\begin{center}
    \includegraphics[width=0.7\textwidth]{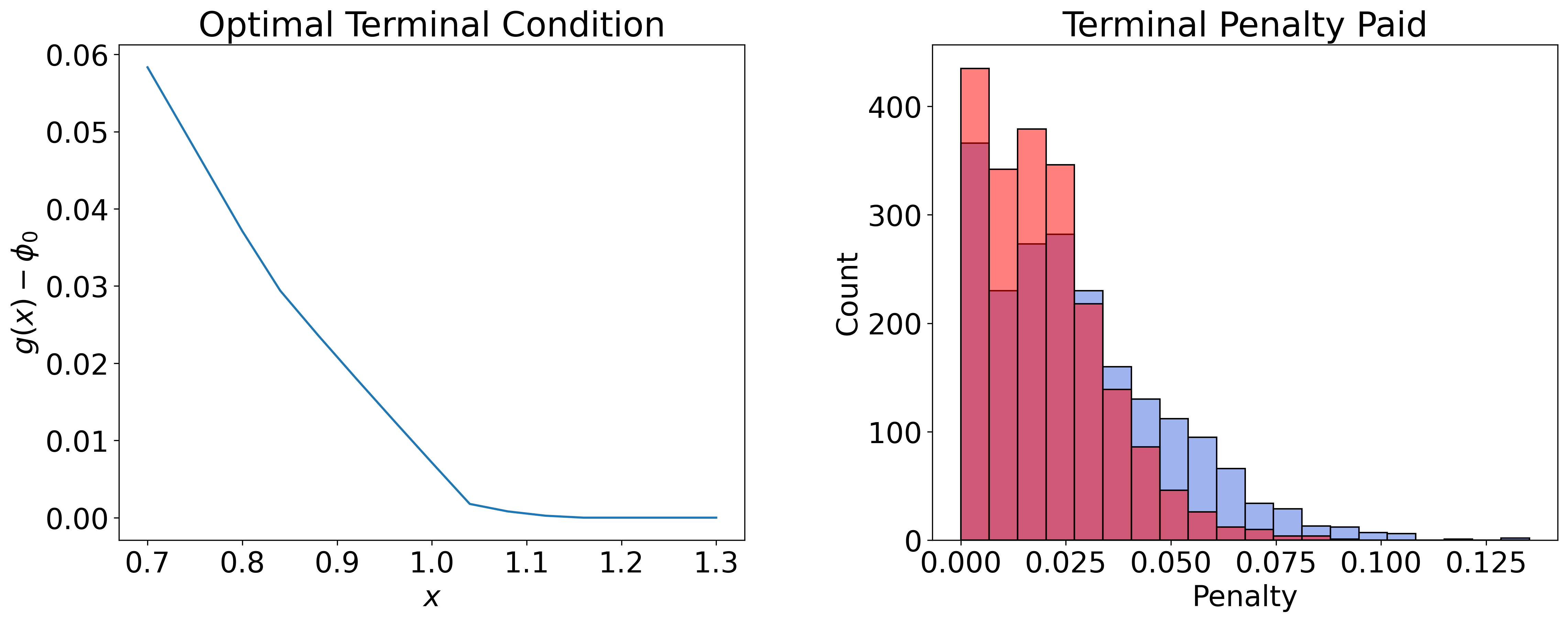}
\end{center}
\caption{Optimal terminal condition (left) and distribution of non-compliance costs paid by the agents (right). Population 1 indicated in red, population 2 indicated in blue. Parameters as in \Cref{tab:compliance.parameters} and \Cref{tab:model.parameters}. Knots as in \Cref{tab:nodes.and.weights}. Parameters as in \Cref{tab:compliance.parameters} and \Cref{tab:model.parameters}.}
\label{fig:10.node.penalty}
\end{figure}

There are many other quantities of interest with respect to the two sub-populations, such as their controls and state processes. In particular, we are interested in how they evolve across the period, as well as their distributions at $T$.  

In \Cref{fig:10.node.inventory.plots}, we plot the sampled trajectories of the inventories for representative agents in each sub-population, as well as a histogram of their terminal inventory, with the non-negligible knot points indicated by dashed lines. In both sub-populations, a considerable number of firms pay the highest marginal penalty for insufficient inventory, indicated by the notable probability mass to the left of the vertical line $x = 0.8$ in the right panel. We also include various terminal inventory percentiles for each sub-population in \Cref{tab:precentiles.terminal.inventories}.

\begin{figure}[ht]
\begin{center}
    \includegraphics[width=0.7\textwidth]{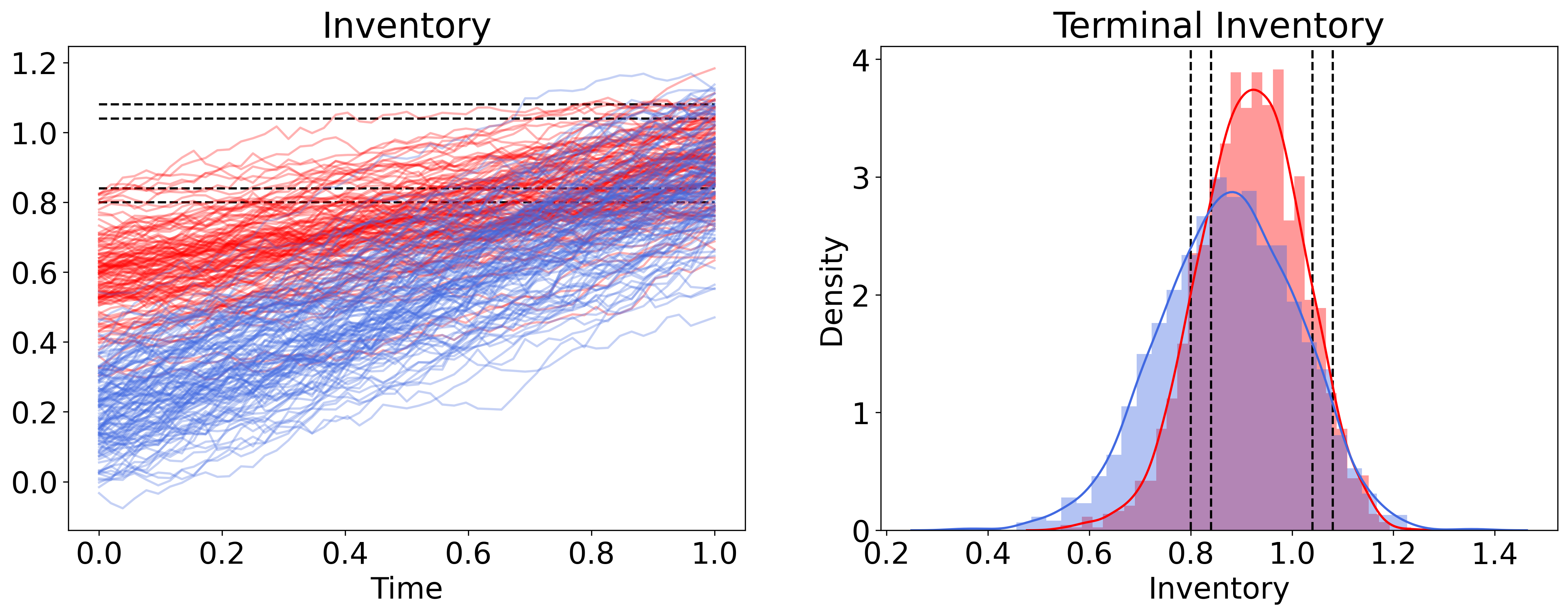}
\end{center}
 \vspace*{-3mm}
\caption{Inventory paths (left) and terminal inventory (right). Population 1 indicated in red, population 2 indicated in blue. The lines denote the non-negligible nodes in the terminal penalty $g$.}
\label{fig:10.node.inventory.plots}
\end{figure}

\begin{table}[ht]

\centering
\begin{tabular}[t]{cccccc}
\toprule\toprule
$\text{Sub-population}$ & $10$-th percentile & $25$-th percentile & $50$-th percentile& $75$-th percentile & $90$-th percentile\\
\hline
$k=1$ &$0.79$&$0.85$& $0.92$ & $0.99$ &$1.05$\\
\hline
$k=2$ &$0.7$&$0.78$& $0.88$ & $0.97$ &$1.05$\\
\toprule\toprule\\
\end{tabular}
 \vspace*{-3mm}
\caption{Percentiles of the terminal inventories across sub-populations. Parameters as in \Cref{tab:compliance.parameters} and \Cref{tab:model.parameters}.}
\label{tab:precentiles.terminal.inventories}
\end{table}%

Interestingly, we observe from \Cref{fig:10.node.inventory.plots} that representative agents from population 1 had a significantly higher average initial inventory than population 2, but their terminal inventories are far closer. This occurs for a variety of reasons. For one, representative agents from population 2 are induced more to acquire RECs due to the nature of the compliance penalty. They also have a much greater initial baseline of REC generation ($h_t^k)$, and a lower rental cost ($\zeta^k)$, allowing them to acquire RECs more quickly.

By plotting the measure flow for each sub-population across time, as in \Cref{fig:initial.and.terminal.inventory}, we can see this behavior more clearly. We can also see that the mean of the inventory distribution for population $1$ shifts to the right as time progressed, with the standard deviation of the distribution appearing fairly constant. On the other hand, the mean of the inventory distribution for population $2$ shifts to the right more aggressively as time goes to $1$, with the standard deviation of the distribution increasing. This suggests greater variance in the REC acquisition rate of a representative firm from population $2$. 

\begin{figure}[ht]
\begin{center}
    \includegraphics[width=0.7\textwidth]{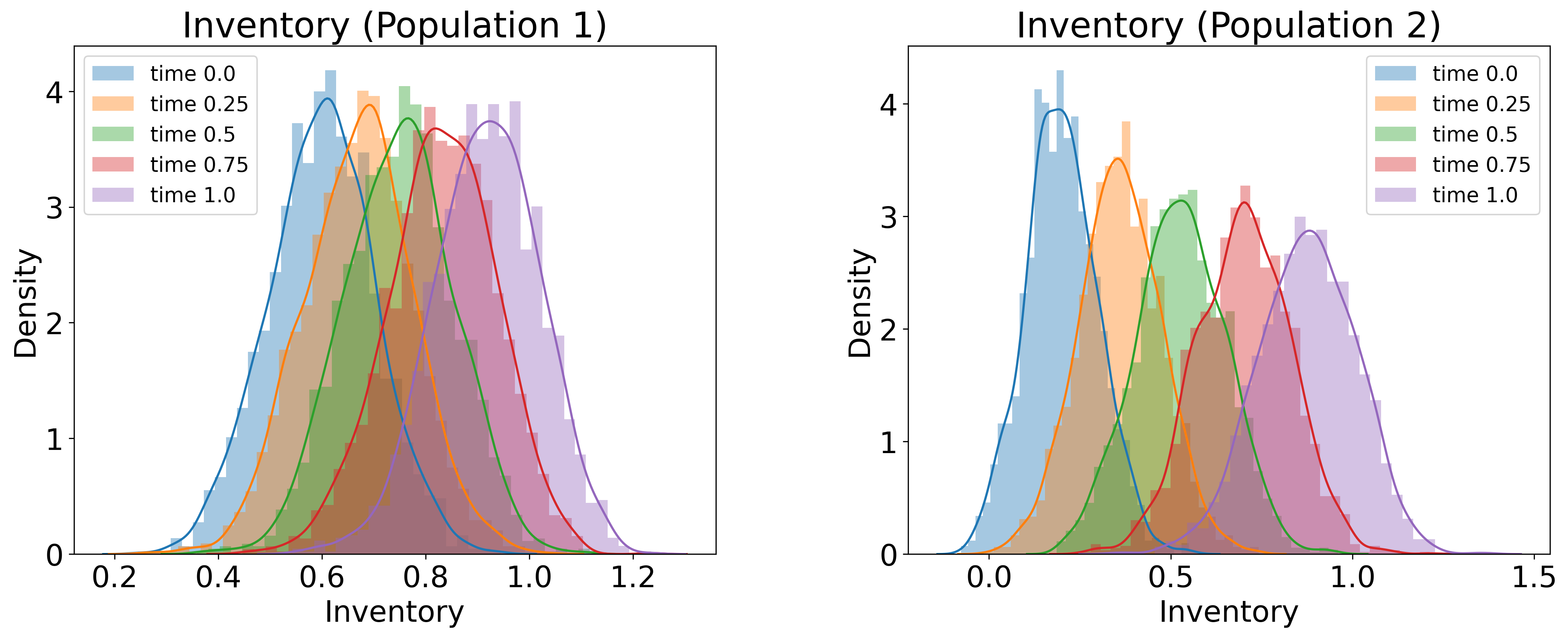}
\end{center}
 \vspace*{-3mm}
\caption{Inventories for population $1$ (left) and population $2$ (right) throughout time. Parameters as in \Cref{tab:compliance.parameters} and \Cref{tab:model.parameters}.}
\label{fig:initial.and.terminal.inventory}
\end{figure}

We now next investigate the agents' optimal controls. Firms may obtain RECs in one of three ways: rental generation, expansion generation, and trading. We wish to understand which of these primarily drove REC acquisitions, and to that end, we plot the acquisition rates and totals for each sub-population across each of these methods in \Cref{fig:inventory.and.expansion}. 

\begin{figure}[ht]
\begin{center}
    \includegraphics[width=0.7\textwidth]{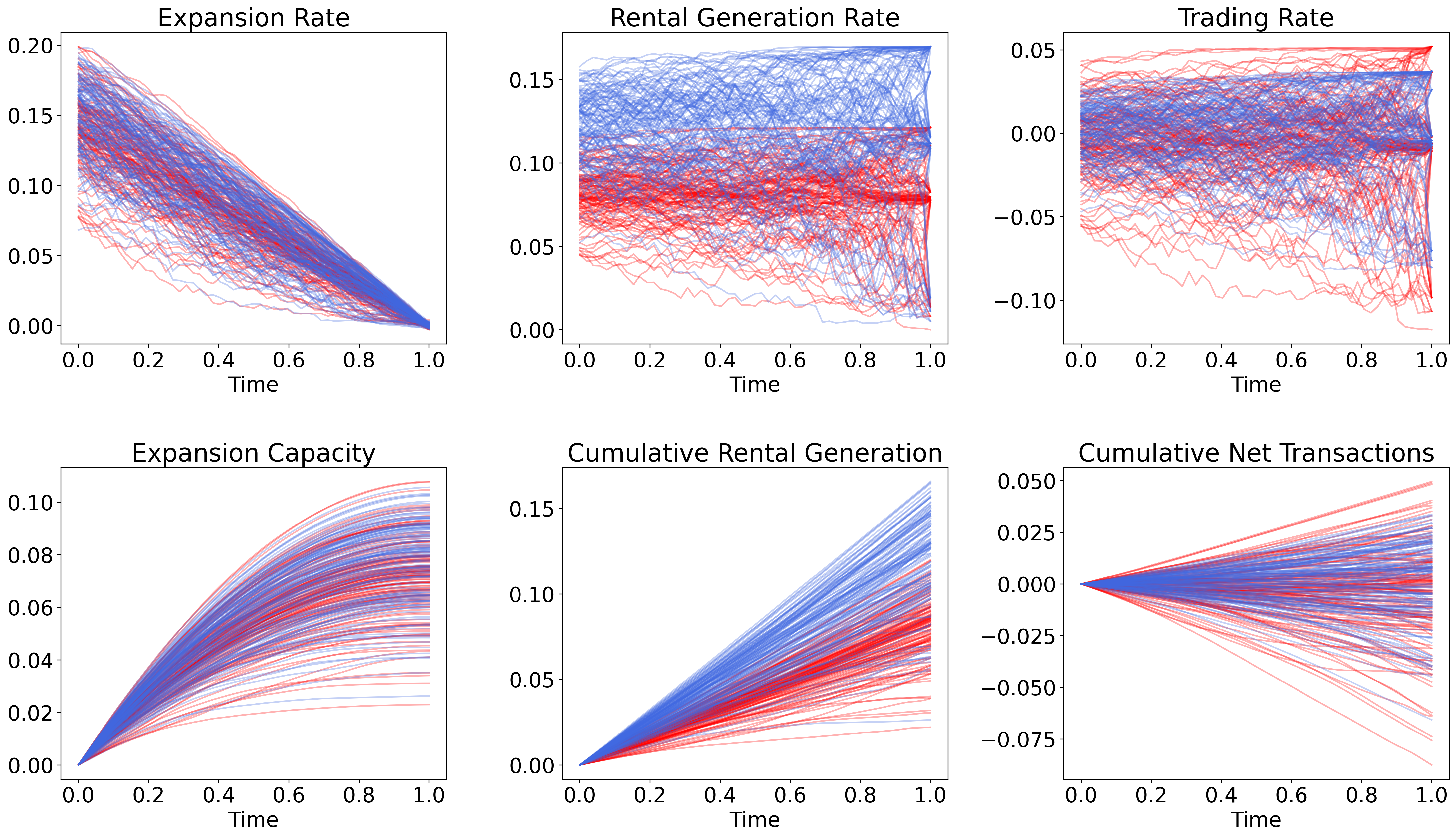}
\end{center}
 \vspace*{-3mm}
\caption{Expansion rate and total expansion (left column), rental generation rate and cumulative rental generation (middle column), trading rate and net trading position (right column) across sampled paths of representative agent from each sub-population. Population 1 indicated by red, population 2 indicated by blue. Parameters as in \Cref{tab:compliance.parameters} and \Cref{tab:model.parameters}.}
\label{fig:inventory.and.expansion}
\end{figure}

From the left panels of \Cref{fig:inventory.and.expansion}, we observe expansion rates for each sub-population taper off to 0 as time progresses, as expected. Firms accrue less benefit from expanding as the terminal time approaches, as they have less time to reap the fruits of their labor. We find the sampled representative firms from sub-population 2 generally expand more than those from sub-population 1, despite the fact that they have the same cost parameters associated with expansion ($\beta^k$). Again, this is expected due to their generally lower initial inventory.

From the middle panels, we see similar behaviour with rental generation. The rental generation rates for population $2$ were greater than the generation rates for population $1$ throughout the time horizon, resulting in greater cumulative rental generation. This is expected due to the lower cost parameter firms from population 2 have for rental generation, and their lower initial inventory incentivizing them to acquire RECs quickly.

From the right panels, we observe the net trading volume for population $2$ was bigger than that of population $1$. This means representative firms from population $1$ were more likely to be net buyers of RECs than representative firms from population $0$. Taken together, these plots provide reasoning for and internal consistency with the previously exemplified behavior of a representative firm from sub-population $2$ generally acquiring more RECs than one from sub-population $1$.

\begin{figure}[ht]
\begin{center}
    \includegraphics[width=0.6\textwidth]{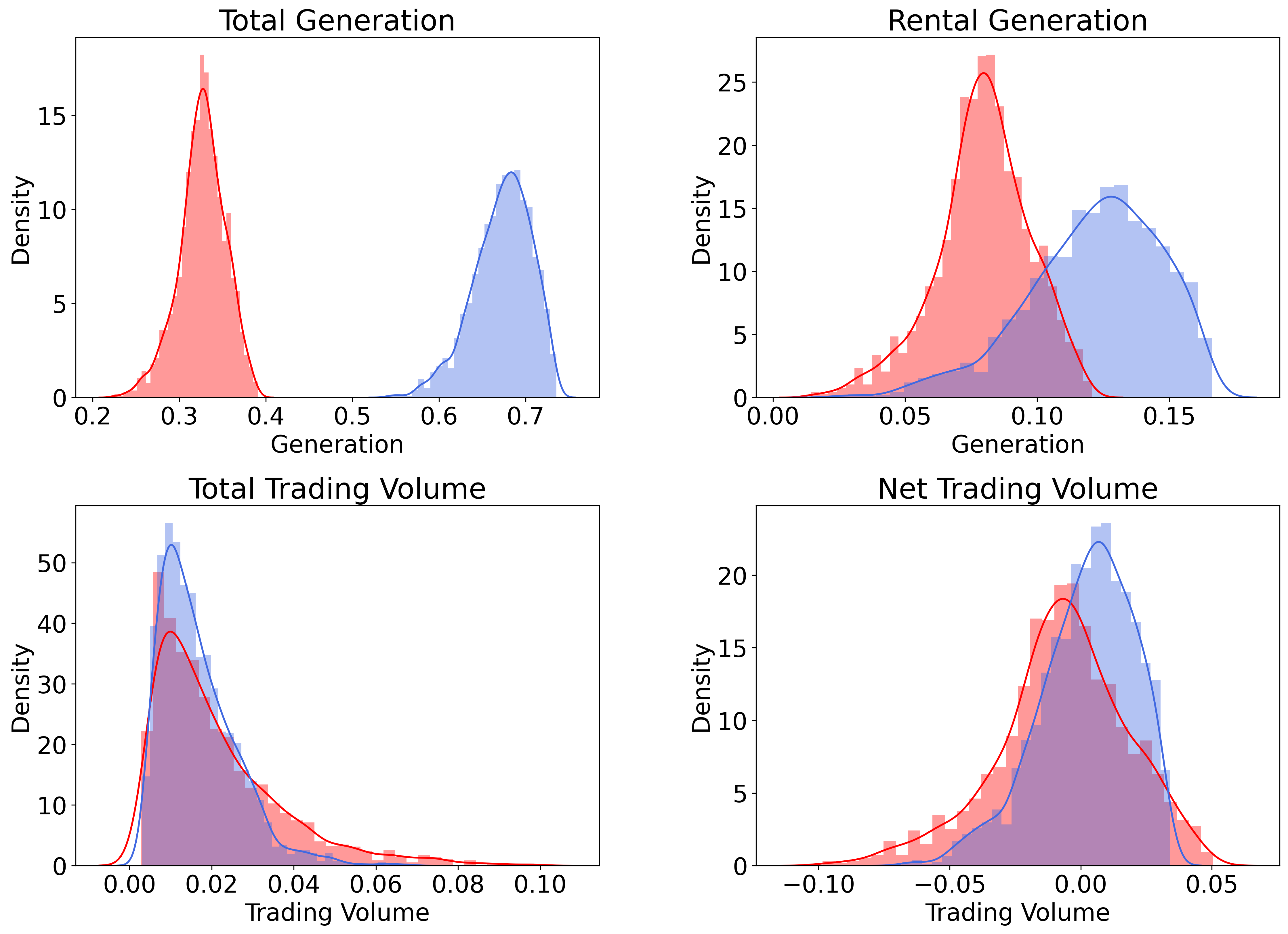}
\end{center}
 \vspace*{-3mm}
\caption{Total generation(top left), rental generation(top right), total trading volume (bottom left) and net trading volume (bottom right) for Population $1$ in red and Population $2$ in blue.}
\label{fig:total.generation.and.trading}
\end{figure}

Figure \Cref{fig:total.generation.and.trading} plots the distribution across samples of representative firms from each sub-population of total generation, rental generation, total trading volume, and rental trading volume. The increased variance in total generation for population $2$ is apparent, which is consistent with the variance of the inventory of a representative firm from said population increasing in time, as seen in \Cref{fig:initial.and.terminal.inventory}.

Finally, with the same argument as in  \cite{shrivats2020mean}, one can show that the equilibrium price process is constant. This fact is observed in \Cref{fig:price.path}, and can be intuited by noting in \eqref{eqn:price.path.sol} that the equilibrium price is in terms of the expectation of the adjoint processes $(Y_t^{(k),X})_{t\in\mathfrak{T}}$ which are, in fact, martingales.
 
\begin{figure}[ht]
\begin{center}
    \includegraphics[width=0.3\textwidth]{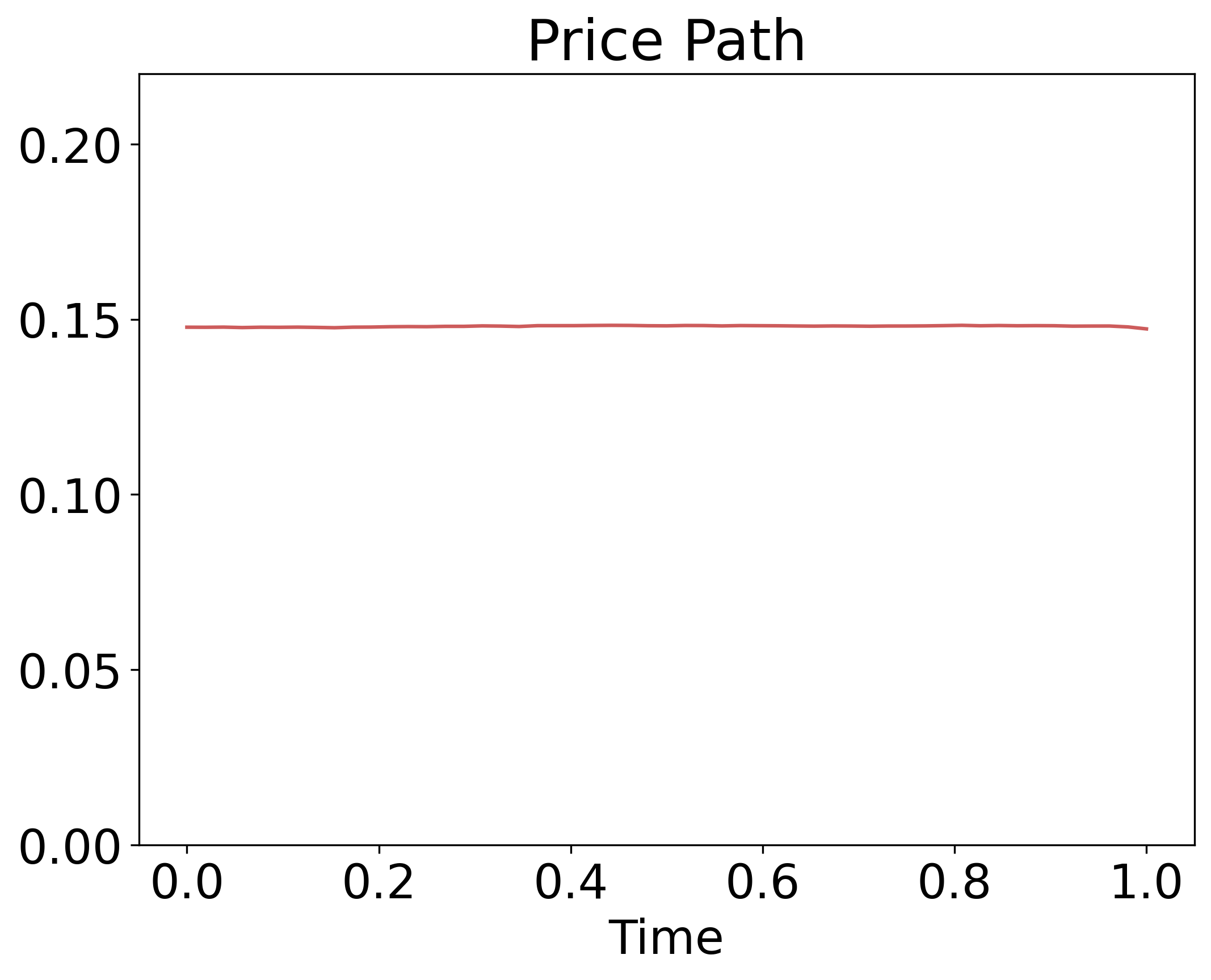}
\end{center}
 \vspace*{-3mm}
\caption{Approximate equilibrium price path.}
\label{fig:price.path}
\end{figure}

\FloatBarrier

\section{Conclusion}
In this paper we introduced a multi-step DL algorithm for solving Principal-Agent mean field Games (PA-MFGs) and illustrated its use on a previously intractable problem in REC markets that involves market clearing. This application demonstrates the flexibility of our method and leads to several insights into the nature of  equilibria in REC markets which are of independent interest. This study also suggests several future directions which we describe below.
\begin{enumerate}
	\item[(i)] This work leaves open the question of convergence guarantees for the numerical algorithm presented in Section \ref{sec:algorithm}. As in the analysis conducted in \cite{carmona2019convergence,carmona2021convergence}, it is of interest to understand how well, and under what assumptions, our numerical scheme approximates the solution of the original PA-MFG problem. \cite{han2020deep} lays a theoretical foundation for the deep BSDE method in the general case of coupled FBSDEs, and show that the error converges to zero given the universal approximation capability of NNs. This could be a good starting point for convergence results of our proposed algorithm, as our algorithm for PA-MFG problem is an alternating optimization whose inner loop involves solving coupled FBSDEs with a similar deep neural net approach.
	\item[(ii)] In Section \ref{sec:example} we introduce a stylistic model for SREC markets. In the setting we provide, we only consider a single compliance period. In practice, firms and regulators in these markets will interact over consecutive periods. Allowing for interaction across these periods through the banking of REC inventories is an important extension that will bridge the gap between the existing single period theory, and the multi-period setting that arises in practice. One clear change that we anticipate is that the rate at which agents expand their capacity will no longer be curtailed at the end of the first period as the new capacity can now be carried through time. This allows agents to ``plan ahead" and anticipate their obligations in future periods.
\end{enumerate}
Nonetheless, the problem formulation, algorithm, and example presented in this work represent a clear step forward within the exciting field of PA-MFGs and their solutions, as well as their tantalizing applicability to the world of environmental regulation and beyond.


\bibliographystyle{siam}%
\bibliography{references}

\end{document}